\documentclass[10pt,twocolumn,letterpaper]{article}

\usepackage{flushend}
\usepackage{cvpr}
\usepackage{times}
\usepackage{epsfig}
\usepackage{graphicx}
\usepackage{stfloats} % for the first figure
\usepackage{multirow}
\usepackage{hhline}
\usepackage{wrapfig}
\usepackage{amsthm}
\usepackage{amsmath}
\usepackage{amssymb}
\usepackage[table,x11names,dvipsnames,table]{xcolor}

\usepackage{caption,subcaption}
\usepackage{tablefootnote}

\usepackage[numbers,sort,compress]{natbib}

\definecolor{our-blue}{rgb}{0.367, 0.504, 0.707}
\definecolor{our-green}{rgb}{0.56, 0.692, 0.195}
\definecolor{our-darkgreen}{rgb}{0.297, 0.348, 0.105}
\definecolor{our-yellow}{rgb}{0.881, 0.611, 0.142}
\definecolor{our-red}{rgb}{0.923, 0.386, 0.209}

\makeatletter 
\renewcommand\paragraph{\@startsection{paragraph}{4}{\z@}%
{6pt plus 2pt minus 2pt}{-1em}
{\normalfont\normalsize\bfseries}} 
\makeatother

\usepackage[
	pagebackref=true,
	breaklinks=true,
	letterpaper=true,
	bookmarks=false,
	colorlinks=true,
	linkcolor=our-blue,
	urlcolor=our-blue,
	citecolor=our-red
]{hyperref}

\cvprfinalcopy % *** Uncomment this line for the final submission

\def\cvprPaperID{6152} % *** Enter the CVPR Paper ID here

\newcommand{\method}{RaMBO\xspace}

\newcommand{\pascal}{Pascal\xspace}
\newcommand{\vocvii}{VOC 07\xspace}
\newcommand{\vocxii}{VOC 12\xspace}
\newcommand{\vocboth}{VOC 07+12\xspace}

\newcommand\Tstrut{\rule{0pt}{2.5ex}}
\newcommand{\R}{\mathbb{R}}
\newcommand{\N}{\mathbb{N}}

\newcommand{\CC}{\mathcal C}
\newcommand{\DD}{\mathcal D}

\newcommand{\y}{\mathbf{y}}

\renewcommand{\d}{\mathrm{d}}

\let\bar\overline

\DeclareMathOperator{\rank}{\mathbf{rk}}
\DeclareMathOperator{\ap}{\mathit{AP}}
\DeclareMathOperator{\map}{\mathit{mAP}}

\DeclareMathOperator{\ndcg}{\mathit{NDCG}}

\DeclareMathOperator*{\argmin}{arg\,min}
\DeclareMathOperator{\rel}{\mathrm{rel}}
\DeclareMathOperator{\precision}{\mathrm{Prec}}
\DeclareMathOperator{\rak}{\mathit{r@K}}
\DeclareMathOperator{\rtak}{\mathit{\widetilde{r}@K}}
\DeclareMathOperator{\Rak}{\mathit{R@K}}
\DeclareMathOperator{\lossk}{\mathit{L@K}}
\DeclareMathOperator{\lrec}{\mathit{L}_{\text{rec}}}
\DeclareMathOperator{\lap}{\mathit{L}_{\ap}}
\DeclareMathOperator{\lmap}{\mathit{L}_{\map}}
\DeclareMathOperator{\lapc}{\mathit{L}_{\ap\CC}}
\DeclareMathOperator*{\E}{\mathbb{E}}

\newcommand{\fig}[1]{Fig.~\ref{#1}}  % fig ref everywhere else
\newcommand{\Tab}[1]{Table~\ref{#1}} % for the beginning of the sentence
\newcommand{\tab}[1]{Tab.~\ref{#1}}

\newcommand{\eqn}[1]{Eq.~\ref{#1}} % equation 1.1
\newcommand{\eqnp}[1]{(Eq.~\ref{#1})} % (equation 1.1)
\renewcommand{\sec}[1]{Sec.~\ref{#1}} % Section 1

\newcounter{problem}
\newcommand{\newproblem}[1]{\refstepcounter{problem}\label{#1}}

\theoremstyle{plain}
\newtheorem{theorem}{Theorem}
\newtheorem{lemma}{Lemma}
\newtheorem{prop}{Proposition}

\usepackage{algorithm}
\usepackage[noend]{algpseudocode}
\algrenewcommand{\algorithmiccomment}[1]{\bgroup\hfill//~#1\egroup}
\algrenewcommand{\Return}{\State\textbf{return}\ }
\algnewcommand{\Save}{\State\textbf{save}\ }
\algnewcommand{\Load}{\State\textbf{load}\ }
\algnewcommand{\Define}{\State\textbf{define}\ }

\ifcvprfinal\fi
\begin{document}

\title{Optimizing Rank-based Metrics with Blackbox Differentiation}

\author{%
	Michal Rol\'\i nek$^1$%
	\thanks{These authors contributed equally.}\ \ ,
	V\'\i t Musil$^{2\,*}$,
	Anselm Paulus$^1$,
	Marin Vlastelica$^1$,
	Claudio Michaelis$^3$,
	Georg Martius$^1$
	\smallskip
	\and
	\normalsize
	$^1$ Max-Planck-Institute for Intelligent Systems,\\
	\normalsize
	T\"ubingen, Germany
	\and
	\normalsize
	$^2$ Universit\`a degli Studi di Firenze,\\
	\normalsize
	Italy
	\and
	\normalsize
	$^3$ University of T\"ubingen,\\
	\normalsize
	Germany
	\and
	\texttt{\small michal.rolinek@tuebingen.mpg.de}
}

\maketitle
\thispagestyle{empty}

\begin{figure*}[b]
  \centering
  \includegraphics[width=0.85\linewidth]{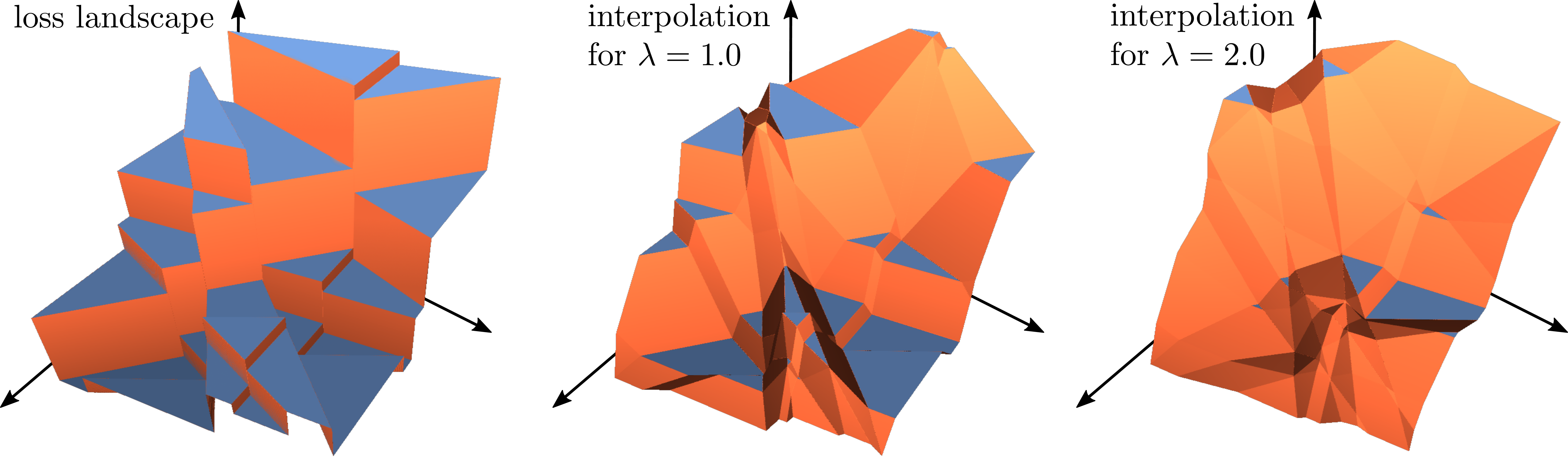}
	\caption{Differentiation of a piecewise constant rank-based loss.  A
	two-dimensional section of the loss landscape is shown (left) along with two
	efficiently differentiable interpolations of increasing strengths (middle and
	right).
  }
  \label{fig:f-lambda-2d}
\end{figure*}

\begin{abstract}
\noindent
Rank-based metrics are some of the most widely used criteria for performance
evaluation of computer vision models. Despite years of effort, direct
optimization for these metrics remains a challenge due to their
non-differentiable and non-decomposable nature. We present an efficient,
theoretically sound, and general method for differentiating rank-based metrics
with mini-batch gradient descent. In addition, we address optimization
instability and sparsity of the supervision signal that both arise from using
rank-based metrics as optimization targets. Resulting losses based on recall
and Average Precision are applied to image retrieval and object detection
tasks. We obtain performance that is competitive with state-of-the-art on
standard image retrieval datasets and consistently improve performance of near
state-of-the-art object detectors.
\end{abstract}

\section{Introduction}

\noindent
Rank-based metrics are frequently used to evaluate performance on a wide
variety of computer vision tasks.  For example, in the case of image retrieval,
these metrics are required since, at test-time, the models produce a ranking of
images based on their relevance to a query.  Rank-based metrics are also
popular in classification tasks with unbalanced class distributions or multiple
classes per image. One prominent example is object detection, where an average
over multiple rank-based metrics is used for final evaluation.  The most common
metrics are recall \cite{ge2018deep}, Average Precision ($\ap$)
\cite{yue2007support}, Normalized Discounted Cumulative Gain ($\ndcg$)
\cite{chakrabarti2008structured}, and the Spearman Coefficient
\cite{cohendet2018media}.

Directly optimizing for the rank-based metrics is inviting but also notoriously
difficult due to the \textit{non-differentiable} (piecewise constant) and
\textit{non-decomposable} nature of such metrics. A trivial solution is to use
one of several popular surrogate functions such as 0-1 loss \citep{Lin2002},
the area under the ROC curve \citep{Bartell:1994} or cross entropy.  Many
studies from the last two decades have addressed direct optimization with
approaches ranging from histogram binning approximations
\cite{cakir2019deep,revaud2019learning, he2018hashing}, finite difference
estimation \cite{henderson2016end}, loss-augmented inference
\cite{yue2007support, mohapatra2018efficient}, gradient approximation
\cite{song2016training} all the way to using a large LSTM to fit the ranking
operation \cite{engilberge2019sodeep}.

Despite the clear progress in direct optimization
\citep{mohapatra2018efficient, cakir2019deep, chen2019towards}, these methods
are notably omitted in the most publicly used implementation hubs for object
detection \citep{mmdetection, detectron2, mrcnn-benchmark,
tf-object-detection-api}, and image retrieval \citep{metriclearning-repo}. The
reasons include poor scaling with sequence lengths, lack of publicly available
implementations that are efficient on modern hardware, and fragility of the
optimization itself.

In a clean formulation, backpropagation through rank-based losses reduces to
providing a meaningful gradient of the piecewise constant ranking function.
This is an interpolation problem, rather than a gradient estimation problem
(the true gradient is simply zero almost everywhere). Accordingly, the
properties of the resulting interpolation (whose gradients are returned) should
be of central focus, rather than the gradient itself.

In this work, we interpolate the ranking function via \emph{blackbox
backpropagation} \cite{blackbox-diff}, a framework recently proposed in the
context of combinatorial solvers. This framework is the first one to give
mathematical guarantees on an interpolation scheme. It applies to piecewise
constant functions that originate from minimizing a discrete objective
function. To use this framework, we reduce the ranking function to a
combinatorial optimization problem. In effect, we inherit two important
features of \cite{blackbox-diff}: \textbf{mathematical guarantees} and the
ability to compute the gradient only with the use of a non-differentiable
\textbf{blackbox implementation} of the ranking function. This allows using
implementations of ranking functions that are already present in popular
machine learning frameworks which results in straightforward implementation and
significant practical speed-up. Finally, differentiating directly the ranking
function gives additional flexibility for designing loss functions.

Having a conceptually pure solution for the differentiation, we can then focus
on another key aspect: sound loss design. To avoid ad-hoc modifications, we
take a deeper look at the caveats of direct optimization for rank-based
metrics. We offer multiple approaches for addressing these caveats, most
notably we introduce  \textbf{margin-based versions} of rank-based losses and
mathematically derive a \textbf{recall-based loss function} that provides dense
supervision.

Experimental evaluation is carried out on image retrieval tasks where we
optimize the recall-based loss and on object detection where we directly
optimize mean Average Precision. On the retrieval experiments, we achieve
performance that is on-par with state-of-the-art while using a simpler setup.
On the detection tasks, we show consistent improvement over highly-optimized
implementations that use the cross-entropy loss, while our loss is used in an
out-of-the-box fashion. We release the code used for our
experiments\footnote{\url{https://github.com/martius-lab/blackbox-backprop}.}.

\section{Related work}

\paragraph{Optimizing for rank-based metrics} As rank-based evaluation metrics
are now central to multiple research areas, their direct optimization has
become of great interest to the community.  Traditional approaches typically
rely on different flavors of loss-augmented inference
\cite{mohapatra2018efficient, yue2007support, mohapatra2014efficient,
mcfee2010metric},  or gradient approximation \cite{song2016training,
henderson2016end}. These approaches often require solving a combinatorial
problem as a subroutine where the nature of the problem is dependent on the
particular rank-based metric. Consequently, efficient algorithms for these
subproblems were proposed \cite{mohapatra2018efficient, song2016training,
yue2007support}.

More recently, differentiable histogram-binning approximations
\cite{cakir2019deep, he2018hashing, he2018local, revaud2019learning} have
gained popularity as they offer a more flexible framework.  Completely
different techniques including learning a distribution over rankings
\cite{taylor2008softrank}, using a policy-gradient update rule
\cite{Rao2018LearningGO}, learning the sorting operation entirely with a deep
LSTM \cite{engilberge2019sodeep} or perceptron-like error-driven updates have
also been applied \cite{chen2019towards}.

\paragraph{Metric learning} There is a great body of work on metric learning
for retrieval tasks, where defining a suitable loss function plays an essential
role. \citet{bellet2013survey} and \citet{kulis2013metric} provide a broader
survey of metric learning techniques and applications.  Approaches with local
losses range from employing pair losses \cite{bromley1994signature,
koch2015siamese}, triplet losses
\cite{schroff2015facenet,hoffer2015deep,song2016training} to quadruplet losses
\cite{law2013quadruplet}.  While the majority of these works focus on local,
decomposable losses as above, multiple lines of work exist for directly
optimizing global rank-based losses \cite{engilberge2019sodeep,
taylor2008softrank, revaud2019learning}.  The importance of good batch sampling
strategies is also well-known, and is the subject of multiple studies
\cite{oh2016deep, schroff2015facenet, ge2018deep, wu2017iccv}, while others
focus on generating novel training examples
\cite{zhao2018adversarial,song2016training,movshovitz-attias2017iccv}.

\paragraph{Object detection} Modern object detectors use a combination of
different losses during training \cite{girshick2014rich, ren2015faster,
liu2015ssd, Redmon_2016_CVPR, he2017mask, lin2018focal}. While the biggest
performance gains have originated from improved architectures
\cite{ren2015faster, he2017mask, Redmon2018YOLOv3, ghiasi2019fpn} and feature
extractors \cite{he2016deep, zoph2018learning}, some works focused on
formulating better loss functions \cite{lin2018focal,
rezatofighi2019generalized, Goldman_2019_CVPR}. Since its introduction in the
Pascal \texttt{VOC} object detection challenge \cite{everingham2010pascal}
\emph{mean Average Precision} ($\map$) has become the main evaluation metric
for detection benchmarks. Using the metric as a replacement for other less
suitable objective functions has thus been studied in several works
\cite{song2016training, henderson2016end, Rao2018LearningGO, chen2019towards}.

\section{Background}
\subsection{Rank-based metrics}

\noindent
For a positive integer $n$, we denote by $\Pi_n$ the set of all permutations of
$\{1,\ldots,n\}$. The rank of vector $\y=[y_1,\dots,y_n]\in\R^n$, denoted by
$\rank(\y)$, is a permutation $\pi\in\Pi_n$ satisfying
\begin{equation} \label{eq:rank-def}
	y_{\pi^{-1}(1)} \ge y_{\pi^{-1}(2)} \ge \cdots \ge y_{\pi^{-1}(n)},
\end{equation}
\ie sorting $\y$.  Note, that rank is not defined uniquely for those vectors
for which any two components coincide. In the formal presentation, we reduce
our attention to \emph{proper rankings} in which ties do not occur.

The rank $\rank$ of the $i$-th element is one plus the number of members in the
sequence exceeding its value, \ie
\begin{equation}\label{eq:rank-def2}
	\rank(\y)_i = 1 + |\{j: y_j > y_i\}|.
\end{equation}

\subsubsection{Average Precision}

\noindent
For a fixed query, let $\y\in\R^n$ be a vector of relevance scores of $n$
examples. We denote by $\y^*\in\{0,1\}^n$ the vector of their ground truth
labels (relevant/irrelevant) and by
\begin{equation} \label{eq:rel-def}
	\rel(\y^*) = \{i:y^*_i = 1\}
\end{equation}
the set of indices of the relevant examples.  Then Average Precision is given
by
\begin{equation} \label{eq:AP-def}
	\ap(\y, \y^*)
		=	\frac{1}{|\rel(\y^*)|} \sum_{i \in\rel(\y^*)} \precision(i),
\end{equation}
where precision at $i$ is
defined as
\begin{equation}
	\precision(i)
		= \frac{|\{j\in\rel(\y^*): y_j \ge y_i\}|}{\rank(\y)_i}
\end{equation}
and describes the ratio of relevant examples among the $i$ highest-scoring
examples.

In classification tasks, the dataset typically consists of annotated images.
This we formalize as pairs $(x_i,\y^*_i)$ where $x_i$ is an input image and
$\y^*_i$ is a binary class vector, where, for every $i$, each
$(\y^*_i)_c\in\{0,1\}$ denotes whether an image $x_i$ belongs to the class
$c\in\CC$.  Then, for each example $x_i$ the model provides a vector of
suggested class-relevance scores $\y_i = \phi(x_i, \theta)$, where $\theta$ are
the parameters of the model.

To evaluate mean Average Precision ($\map$), we consider for each class
$c\in\CC$ the vector of scores $\y(c)=[(\y_i)_c]_{i}$ and labels
$\y^*(c)=[(\y^*_i)_c]_i$.  We then take the mean of Average Precisions over all
the classes
\begin{equation} \label{eq:mAP-def}
	\map
		= \frac{1}{|\CC|} \sum_{c \in \CC}
			\ap\bigl(\y(c), \y^*(c)\bigr).
\end{equation}

Note that $\map \in [0, 1]$ and that the highest score 1 corresponds to perfect
score prediction in which all relevant examples precede all irrelevant
examples.

\subsubsection{Recall}

\noindent
Recall is a metric that is often used for information retrieval.  Let again
$\y\in\R^n$ and $\y^*\in\{0,1\}^n$ be the scores and the ground-truth labels
for a given query over a dataset.  For a positive integer $K$, we set
\begin{equation} \label{eq:rec-k}
	\rak(\y,\y^*) =
		\begin{cases}
			1 & \text{if $\exists i \in\rel(\y^*)$ with $\rank(\y)_i \le K$}\\
			0 &\text{otherwise,}
		\end{cases}
\end{equation}
where $\rel(\y^*)$ is given in \eqn{eq:rel-def}.

In a setup where each element $x_i$ of the dataset $\DD$ is a possible query,
we define the ground truth matrix as follows. We set $\y^*_i(j)=1$ if $x_j$
belongs to the same class as the query $x_i$, and zero otherwise.  The scores
suggested by the model are again denoted by $\y_i = [\phi(x_i,x_j,\theta)\colon
j \in \DD]$.

In order to evaluate the model over the whole dataset $\DD$, we average $\rak$
over all the queries $x_i$, namely
\begin{equation}
	\Rak
		= \frac{1}{|\DD|} \sum_{i\in\DD}
			\rak\bigl(\y_i,\y^*_i\bigr).
\end{equation}

Again, $\Rak\in[0,1]$ for every $K$.  The highest score 1 means that a relevant
example is always found among the top $K$ predictions.

\subsection{Blackbox differentiation of combinatorial solvers}

\newcommand{\usedtobey}{\mathbf{s}}
\newcommand{\usedtobeY}{{S}}

\noindent
In order to differentiate the ranking function, we employ a method for
efficient backpropagation through combinatorial solvers -- recently proposed in
\citep{blackbox-diff}.  It turns algorithms or solvers for problems like
\textsc{shortest-path}, \textsc{traveling-salesman-problem}, and various graph
cuts into differentiable building blocks of neural network architectures.

With minor simplifications, such solvers (\eg for the \textsc{Multicut}
problem) can be formalized as maps that take continuous input $\y\in \R^n$ (\eg
edge weights of a fixed graph) and return discrete output $\usedtobey
\in\usedtobeY\subset\R^n$ (\eg indicator vector of a subset of edges forming a
cut) such that it minimizes a combinatorial objective expressed as an inner
product $\y\cdot\usedtobey$ (\eg the cost of the cut). Note that the notation
differs from \citep{blackbox-diff} ($\y$ was $w$ and $\usedtobey$ was $y$).  In
short, a blackbox solver is
\begin{equation} \label{E:solver}
	\y \mapsto \usedtobey(\y)
		\quad \text{such that} \quad
	\usedtobey(\y)
		= \argmin_{\usedtobey\in\usedtobeY} \y\cdot\usedtobey,
\end{equation}
where $\usedtobeY$ is the discrete set of admissible assignments (\eg subsets
of edges forming cuts).

The key technical challenge when computing the backward pass is meaningful
differentiation of the piecewise constant function $\y \to L(\usedtobey(\y))$
where $L$ is the final loss of the network. To that end, \citep{blackbox-diff}
constructs a family of continuous and (almost everywhere) differentiable
functions parametrized by a single hyperparameter $\lambda>0$ that controls the
trade-off between ``faithfulness to original function'' and ``informativeness
of the gradient'', see \fig{fig:f-lambda-2d}.  For a fixed $\lambda$ the
gradient of such an interpolation at point $\usedtobey$ is computed and passed
further down the network (instead of the true zero gradient) as
\begin{align}
  \frac{\partial L(\usedtobey(\y))}{\partial \y}
		:= - \frac1\lambda (\usedtobey-\usedtobey_{\lambda})
\end{align}
where $\usedtobey_{\lambda}$ is the output of the solver for a certain
precisely constructed modification of the input. The modification is where the
incoming gradient information $\partial L(\usedtobey(\y))/{\partial
\usedtobey(\y)}$ is used.
For full details including the mathematical guarantees on the tightness of the
interpolation, see \citep{blackbox-diff}.

The main advantage of this method is that only a blackbox implementation of the
solver (\ie of the forward pass) is required to compute the backward pass.
This implies that powerful optimized solvers can be used instead of relying on
suboptimal differentiable relaxations.

\section{Method}

\subsection{Blackbox differentiation for ranking}

\noindent
In order to apply blackbox differentiation method for ranking, we need to find
a suitable combinatorial objective. Let $\y\in\R^n$ be a vector of $n$ real
numbers (the scores) and let $\rank\in\Pi_n$ be their ranks. The connection
between blackbox solver and ranking is captured in the following proposition.

\begin{prop}\label{prop:rank}
In the notation set by Eqs.~\eqref{eq:rank-def} and \eqref{eq:rank-def2}, we
have  
\begin{equation} \label{eq:rank-solver-def}
	\rank(\y) = \argmin_{\pi\in\Pi_n} \y\cdot\pi.
\end{equation}
\end{prop}

In other words, the mapping $\y\to\rank(\y)$ is a minimizer of a linear
combinatorial objective just as \eqn{E:solver} requires.

The proof of Proposition \ref{prop:rank} rests upon a classical rearrangement
inequality \citep[Theorem~368]{hardy1952inequalities}. The following theorem is
its weaker formulation that is sufficient for our purpose.

\begin{theorem}[Rearrangement inequality] \label{T:rearr-ineq}
For every positive integer $n$, every choice of real numbers $y_1 \ge \cdots
\ge y_n$ and every permutation $\pi\in\Pi_n$ it is true that
\begin{equation*}
		y_1 \cdot 1 + \cdots + y_n \cdot n
			\le y_1 \pi(1) + \cdots + y_n \pi(n).
\end{equation*}
Moreover, if $y_1,\dots,y_n$ are distinct, equality occurs precisely for the
identity permutation $\pi$.
\end{theorem}

\begin{proof}[Proof of Proposition~\ref{prop:rank}]
Let $\pi$ be the permutation that minimizes \eqref{eq:rank-solver-def}. This
means that the value of the sum
\begin{equation} \label{eq:ypi-minimal}
	y_1\pi(1) + \cdots + y_n\pi(n)
\end{equation}
is the lowest possible. Using the inverse permutation $\pi^{-1}$
\eqref{eq:ypi-minimal} rewrites as
\begin{equation} \label{eq:ypi-minimal-2}
	y_{\pi^{-1}(1)}\cdot 1 + \cdots + y_{\pi^{-1}(n)}\cdot n
\end{equation}
and therefore, being minimal in \eqref{eq:ypi-minimal-2} makes
\eqref{eq:rank-def} hold due to Theorem \ref{T:rearr-ineq}.  This shows that
$\pi=\rank(y)$.
\end{proof}

The resulting gradient computation is provided in Algorithm \ref{algo:main} and
only takes a few lines of code.  We call the method \emph{Ranking Metric
Blackbox Optimization} (\method{}).

Note again the presence of a blackbox ranking operation. In practical
implementation, we can delegate this to a built-in function of the employed
framework (\eg{} \textsc{torch.argsort}). Consequently, we inherit the $O(n
\log n)$ computational complexity as well as a fast vectorized implementation
on a GPU. To our knowledge, the resulting algorithm is the first to have both
truly sub-quadratic complexity (for both forward and backward pass) \emph{and}
to operate with a general ranking function as can be seen in
\tab{table:computational_complexity} (not however that
\cite{mohapatra2018efficient} have a lower complexity as they specialize on AP
and not general ranking).

\begin{table}[t]
	\centering
	\renewcommand{\arraystretch}{1.15}
	\small
	\begin{tabular*}{\linewidth}{@{\extracolsep\fill}lcc@{\hskip2ex}}
		\hline\Tstrut
	    Method & forward + backward &  general ranking\\
		\hline\Tstrut
	    \textbf{\method} & $O(n \log n)$  & \checkmark \\
	    Mohapatra et al. \cite{mohapatra2018efficient} & $O(n \log p)$  & x \\
	    Chen et al. \cite{chen2019towards} & $O(np)$  & \checkmark \\
	    Yue et al. \cite{yue2007support} & $O(n^2)$ &  x \\
	    FastAP \cite{cakir2019deep} & $O\bigl((n+p)L\bigr)$ & \checkmark \\
	    SoDeep \cite{engilberge2019sodeep} & $O\bigl((n+p)h^2\bigr)$ &  \checkmark \\
		\hline
	\end{tabular*}
	\caption{Computational complexity of different approaches for differentiable
		ranking. The numbers of the negative and of the positive examples are denoted
		by $n$ and $p$, respectively.  For SoDeep, $h$ denotes the LSTM's hidden
		state size ($h \approx n$) and for FastAP $L$ denotes the number of bins.
		\method{} is the first method to directly differentiate general ranking with
		a truly sub-quadratic complexity.}
	\label{table:computational_complexity}
\end{table}

\algrenewcommand{\algorithmiccomment}[1]{\bgroup\hskip2em\textcolor{ourspecialtextcolor}{//~\textsl{#1}}\egroup}
\begin{algorithm}\vspace{.5ex}
\begin{algorithmic}
        \Define \textbf{Ranker} \textbf{as} blackbox operation computing ranks
\end{algorithmic}
\begin{algorithmic}
        \Function{ForwardPass}{$\y$}
        \State $\rank(\y) :=$ \textbf{Ranker}($\y$)
        \Save $ \y$ and $\rank(\y)$ for backward pass
        \Return $\rank(\y)$
        \EndFunction

        \Function{BackwardPass}{$\frac{\d L}{\d\rank}$}
        \Load $\y$ and $\rank(\y)$ from forward pass
        \Load hyperparameter $\lambda$
        \State $\y_\lambda := \y + \lambda \cdot \tfrac{\d L}{\d\rank}$
        \State $\rank(\y_\lambda) :=$ \textbf{Ranker}($\y_\lambda$)
        \Return $-\frac{1}{\lambda}\bigl[ \rank(\y) - \rank(\y_\lambda)\bigr]$
        \EndFunction
\end{algorithmic}
\caption{\Tstrut\method: Blackbox differentiation for ranking}
\label{algo:main}
\end{algorithm}

\subsection{Caveats for sound loss design}

\noindent
Is resolving the non-differentiability all that is needed for direct
optimization?  Unfortunately not. To obtain well-behaved loss functions, some
delicate considerations need to be made. Below we list a few problems
\ref{prob:batchsize}--\ref{prob:sparse} that arise from direct optimization
without further adjustments.

\newproblem{prob:batchsize}
\paragraph{\ref{prob:batchsize}} Evaluation of rank-based metrics is typically
carried out over the whole test set while direct optimization methods rely on
mini-batch approximations. This, however, \textbf{does not yield an unbiased
gradient estimate}. Particularly small mini-batch sizes result in optimizing a
very poor approximation of $\map$, see \fig{fig:map:decomp}.

\newproblem{prob:margin}
\paragraph{\ref{prob:margin}} Rank-based metrics are \textbf{brittle when many
ties happen in the ranking}. As an example, note that any rank-based metric
attains \emph{all its values} in the neighborhood of a dataset-wide tie.
Additionally, once a positive example is rated higher than all negative
examples even by the slightest difference, the metric gives no incentive for
increasing the difference. This induces a high sensitivity to potential shifts
in the statistics when switching to the test set. The need to pay special
attention to ties was also noted in \citep{cakir2019deep, he2018hashing}.

\newproblem{prob:sparse}
\paragraph{\ref{prob:sparse}} Some metrics give only \textbf{sparse
supervision}. For example, the value of $\rak$ only improves if the
\textit{highest-ranked} positive example moves up the ranking, while the other
positives have no incentive to do so. Similarly, Average Precision does not
give the incentive to decrease the possibly high scores of negative examples,
unless also some positive examples are present in the mini-batch. Since
positive examples are typically rare, this can be problematic.

\begin{figure}
    \centering
    \includegraphics[width=0.8\linewidth]{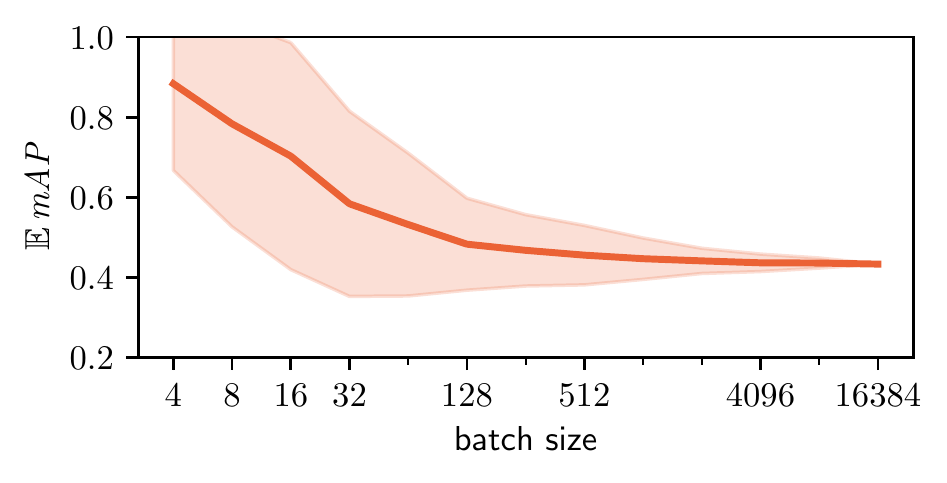}
		\caption{Mini-batch estimation of \emph{mean Average Precision}. The
			expected $\map$ (\ie the optimized loss) is an overly optimistic estimator
			of the true $\map$ over the dataset; particularly for small batch sizes.
			The mean and standard deviations over sampled mini-batch estimates are
			displayed.}
    \label{fig:map:decomp}
\end{figure}

\subsection{Score memory}\label{sec:memory}

\noindent
In order to mitigate the negative impact of small batch sizes on approximating
the dataset-wide loss \ref{prob:batchsize} we introduce a simple running
memory. It stores the scores for elements of the last $\tau$ previous batches,
thereby reducing the bias of the estimate.  All entries are concatenated for
loss evaluation, but the gradients only flow through the current batch.  This
is a simpler variant of ``batch-extension'' mechanisms introduced in
\citep{cakir2019deep, revaud2019learning}. Since only the scores are stored,
and not network parameters or the computational graph, this procedure has a
minimal GPU memory footprint.

\subsection{Score margin}\label{sec:margin}

\noindent
Our remedy for brittleness around ties \ref{prob:margin} is inspired by the
triplet loss \citep{schroff2015facenet}; we introduce a shift in the scores
during training in order to induce a \emph{margin}.  In particular, we add a
negative shift to the positively labeled scores and positive shift the
negatively labeled scores as illustrated in \fig{fig:margin}.  This also
implicitly removes the destabilizing scale-invariance.
Using notation as before, we modify the scores as
\begin{equation} \label{eq:margin}
	\overleftrightarrow{\y}_i =
	\begin{cases}
		y_i + \frac\alpha2
			& \text{if $y^*_i = 0$}
			\\
		y_i - \frac\alpha2
			& \text{if $y^*_i = 1$}
	\end{cases}
\end{equation}
where $\alpha$ is the prescribed margin. In the implementation, we replace the
ranking operation with $\rank_\alpha$ given by
\begin{equation} \label{eq:ranker-margin}
	\rank_\alpha(\y) = \rank\left(\overleftrightarrow{\y}\right).
\end{equation}

\begin{figure}
  \centering
  \includegraphics[width=0.95\linewidth]{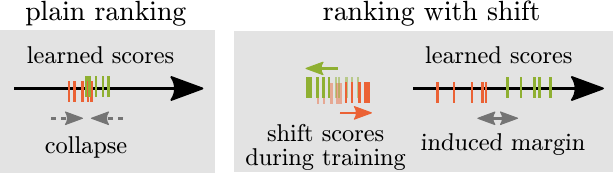}
	\caption{Naive rank-based losses can collapse during optimization. Shifting
		the scores during training induces a margin and a suitable scale for the
		scores. Red lines indicate negative scores and green positive scores.}
  \label{fig:margin}
\end{figure}

\subsection{Recall loss design}

\noindent
Let $\y$ be scores and $\y^*$ the truth labels, as usual.  As noted in
\ref{prob:sparse} the value of $\rak$ only depends on the highest scoring
relevant element.  We overcome the sparsity of the supervision by introducing a
refined metric
\begin{equation} \label{eq:rec-k-refinement}
	\rtak(\y,\y^*)
		= \frac{|\{i\in\rel(\y^*):r_i<K\}|}{|\rel(\y^*)|},
\end{equation}
where $\rel(\y^*)$ denotes the set of relevant elements \eqref{eq:rel-def} and
$r_i$ stands for the number of irrelevant elements outrunning the $i$-th
element. Formally,
\begin{equation} \label{eq:ranks-for-queries}
	r_i = \rank_\alpha(\y)_i - \rank_\alpha(\y^+)_i
		\quad\text{for $i\in\rel(\y^*)$},
\end{equation}
in which $\rank_\alpha(\y^+)_i$ denotes the rank of the $i$-th element only
within the relevant ones.  Note that $\rtak$ depends on all the relevant
elements as intended.  We then define the loss at~$K$ as
\begin{equation} \label{eq:lak-def}
	\lossk(\y,\y^*)
		= 1 - \rtak(\y,\y^*).
\end{equation}
Next, we choose a weighting $w_K\ge 0$ of these losses
\begin{equation} \label{eq:lrec-single}
	\lrec(\y,\y^*)
		= \sum_{K=1}^\infty w_K \lossk(\y,\y^*),
\end{equation}
over values of~$K$.

Proposition~\ref{prop:recall-loss} (see the Supplementary material) computes a
closed form of \eqref{eq:lrec-single} for a given sequence of weights $w_K$.
Here, we exhibit closed-form solutions for two natural decreasing sequences of
weights:
\begin{equation} \label{eq:lrec-log}
	\lrec(\y,\y^*) =
		\begin{cases}
			\displaystyle
			\E_{i\in\rel(\y^*)} \ell(r_i)
				& \text{if $w_K\approx\frac1K$}
				\\
			\displaystyle
			\E_{i\in\rel(\y^*)} \ell\bigl(\ell(r_i)\bigr)
				& \text{if $w_K\approx\frac1{K\log K}$},
		\end{cases}
\end{equation}
where $\ell(k)=\log(1+k)$.

This also gives a theoretical explanation why some previous works
\citep{chen2019towards,henderson2016end} found it ``beneficial'' to optimize
the logarithm of a ranking metric, rather than the metric itself.  In our case,
the $\log$ arises from the most natural weight decay~$1/K$.

\subsection{Average Precision loss design}
\label{subsec:ap_loss}

\noindent
Having differentiable ranking, the generic $\ap$ does not require any further
modifications. Indeed, for any relevant element index $i\in\rel(\y^*)$, its
precision obeys
\begin{equation} \label{eq:prec-rank}
	\precision(i) = \frac{\rank_\alpha(\y^+)_i}{\rank_\alpha(\y)_i}
\end{equation}
where $\rank(\y^+)_i$ is the rank of the $i$-th element within all the relevant
ones.  The $\ap$ loss then reads
\begin{equation} \label{eq:lap}
	\lap(\y,\y^*)
		= 1 - \E_{i\in\rel(\y^*)} \precision(i).
\end{equation}
For calculating the mean Average Precision loss $\lmap$, we simply take the
mean over the classes $\CC$.

To alleviate the sparsity of supervision caused by rare positive examples
\ref{prob:sparse}, we also consider the $\ap$ loss across all the classes.
More specifically, we treat the matrices $\y(c)$ and $\y^*(c)$ as concatenated
vectors $\bar\y$ and $\bar\y^*$, respectively, and set
\begin{equation}
	\lapc = \lap(\bar\y,\bar\y^*).
\end{equation}
This practice is consistent with \cite{chen2019towards}.

\section{Experiments}

\noindent
We evaluate the performance of \method on object detection and several image
retrieval benchmarks.  The experiments demonstrate that our method for
differentiating through $\map$ and recall is generally on-par with the
state-of-the-art results and yields in some cases better performance.  We will
release code upon publication. Throughout the experimental section, the numbers
we report for \method are averaged over three restarts.

\subsection{Image retrieval}

\noindent
To evaluate the proposed Recall Loss \eqnp{eq:lrec-log} derived from \method we
run experiments for image retrieval on the CUB-200-2011 \cite{cub200}, Stanford
Online Products \cite{stanford-products}, and In-shop Clothes \cite{inshop}
benchmarks. We compare against a variety of methods from recent years, multiple
of which achieve state-of-the-art performance. The best-performing methods are
ABE-8 \cite{kim2018attention}, FastAP \cite{cakir2019deep}, and Proxy NCA
\cite{movshovitz-attias2017iccv}.

\begin{figure}
  \centering
  \includegraphics[width=0.95\linewidth]{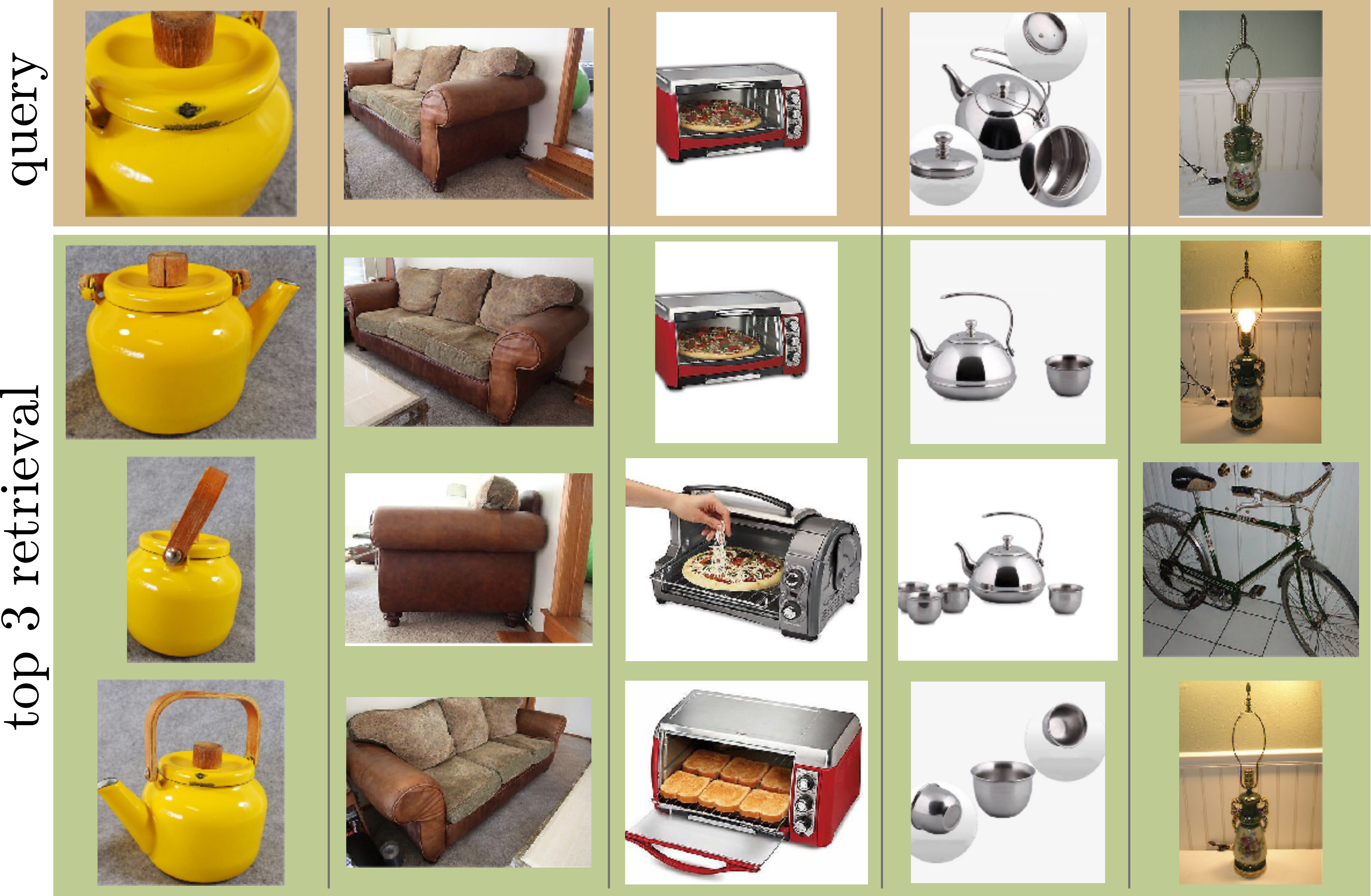}
  \caption{\emph{Stanford Online Products} image retrieval examples.}
  \label{fig:online-prod}
\end{figure}

\paragraph{Architecture}
For all experiments, we follow the most standard setup. We use a pretrained
ResNet50 \cite{he2016deep} in which we replace the final softmax layer with a
fully connected embedding layer which produces a $512$-dimensional vector for
each batch element. We normalize each vector so that it represents a point on
the unit sphere.  The cosine similarities of all the distinct pairs of elements
in the batch are then computed and the ground truth similarities are set to 1
for those elements belonging to the same class and 0 otherwise. The obvious
similarity of each element with itself is disregarded.  We compute the $\lrec$
loss for each batch element with respect to all other batch elements using the
similarities and average it to compute the final loss.  Note that our method
does not employ any sampling strategy for \textit{mining} suitable
pairs/triplets from those present in a batch. It does, however, share a
\textit{batch preparation} strategy with \cite{cakir2019deep} on two of the
datasets.

\paragraph{Parameters}
We use Adam optimizer \cite{kingma2014adam} with an amplified learning rate for
the embedding layer. We consistently set the batch size to 128 so that each
experiment runs on a GPU with 16GB memory.  Full details regarding training
schedules and exact values of hyperparameters for the different datasets are in
the Supplementary material.

\paragraph{Datasets}
For data preparation, we resize images to $256\times 256$ and randomly crop and
flip them to $224\times 224$ during training, using a single center crop on
evaluation.

We use the \emph{Stanford Online Products} dataset consisting of $120,053$
images with $22,634$ classes crawled from Ebay. The classes are grouped into 12
superclasses (\eg cup, bicycle) which are used for mini-batch preparation
following the procedure proposed in \cite{cakir2019deep}. We follow the
evaluation protocol proposed in \cite{stanford-products}, using $59,551$ images
corresponding to $11,318$ classes for training and $60,502$ images
corresponding to $11,316$ classes for testing.

The \emph{In-shop Clothes} dataset consists of $54,642$ images with $11,735$
classes. The classes are grouped into 23 superclasses (\eg MEN/Denim,
WOMEN/Dresses), which we use for mini-batch preparation as before.  We follow
previous work by using $25,882$ images corresponding to $3,997$ classes for
training and $14,218$ + $12,612$ images corresponding to $3,985$ classes each
for testing (split into a query + gallery set respectively). Given an image
from the query set, we retrieve corresponding images from the gallery set.

The \emph{CUB-200-2011} dataset consists of $11,788$ images of $200$ bird
categories.  Again we follow the evaluation protocol proposed in
\cite{stanford-products}, using the first $100$ classes consisting of $5,864$
images for training and the remaining $100$ classes with $5,924$ images for
testing.

\begin{table}
    \renewcommand{\arraystretch}{1.15}
    \centering
    \begin{tabular*}{\linewidth}{@{\ \extracolsep\fill}l|cccc@{\ }}
        \hline\Tstrut
        $\Rak$                                         & 1             & 10            & 100           & 1000          \\
        \hline\Tstrut
        Contrastive$^{512}_G$ \cite{oh2016deep}                  & 42.0          & 58.2          & 73.8          & 89.1          \\
        Triplet$^{512}_G$ \cite{oh2016deep}                      & 42.1          & 63.5          & 82.5          & 94.8          \\
        LiftedStruct$^{512}_G$ \cite{oh2016deep}                 & 62.1          & 79.8          & 91.3          & 97.4          \\
        Binomial Deviance$^{512}_G$ \cite{ustinova2016histogram} & 65.5          & 82.3          & 92.3          & 97.6          \\
        Histogram Loss$^{512}_G$ \cite{ustinova2016histogram}    & 63.9          & 81.7          & 92.2          & 97.7          \\
        N-Pair-Loss$^{512}_G$ \cite{sohn2016improved}            & 67.7          & 83.8          & 93.0          & 97.8          \\
        Clustering$^{64}_G$ \cite{song2017cvpr}                 & 67.0          & 83.7          & 93.2          & -             \\
        HDC$^{384}_G$ \cite{yuan2016hard}                        & 69.5          & 84.4          & 92.8          & 97.7          \\
        Angular Loss$^{512}_G$ \cite{wang2017iccv}               & 70.9          & 85.0          & 93.5          & 98.0          \\
        Margin$^{128}_{R50}$ \cite{wu2017iccv}                       & 72.7          & 86.2          & 93.8          & 98.0          \\
        Proxy NCA$^{64}_G$ \cite{movshovitz-attias2017iccv}     & 73.7          & -             & -             & -             \\
        A-BIER$^{512}_G$ \cite{opitz2018deep}                    & 74.2          & 86.9          & 94.0          & 97.8          \\
        HTL$^{128}_G$ \cite{ge2018deep}                          & 74.8          & 88.3          & 94.8          & 98.4          \\
        ABE-8$^{512}_G$ \cite{kim2018attention}                  & 76.3          & 88.4          & 94.8          & 98.2          \\
        FastAP$^{512}_{R50}$ \cite{cakir2019deep}                    & 76.4          & 89.1          & 95.4          & 98.5          \\
        \hline\Tstrut
        \textbf{\method}$^{512}_{R50}\log$         & 77.8 & 90.1 & 95.9 & \textbf{98.7} \\
        \textbf{\method}$^{512}_{R50}\log\log$         & \textbf{78.6} & \textbf{90.5} & \textbf{96.0} & \textbf{98.7} \\
        \hline
    \end{tabular*}
		\caption{Comparison with the state-of-the-art on the Stanford Online
			Products~\cite{oh2016deep}. On this dataset, with the highest number of
			classes in the test set, \method gives better performance than other
			state-of-the-art methods.}
    \label{tbl:sota-products}
\end{table}

\begin{table}
    \renewcommand{\arraystretch}{1.15}
    \centering
    \begin{tabular*}{\linewidth}{@{\extracolsep\fill}l|cccc@{\ }}
        \hline\Tstrut
        $\Rak$                                            & 1     & 2     & 4    & 8
        \\
        \hline\Tstrut
        Contrastive$^{512}_G$ \cite{oh2016deep}                  & 26.4  & 37.7  & 49.8 & 62.3  \\
        Triplet$^{512}_G$ \cite{oh2016deep}                      & 36.1  & 48.6  & 59.3 & 70.0  \\
        LiftedStruct$^{512}_G$ \cite{oh2016deep}                 & 47.2  & 58.9  & 70.2 & 80.2  \\
        Binomial Deviance$^{512}_G$ \cite{ustinova2016histogram} & 52.8  & 64.4  & 74.7 & 83.9  \\
        Histogram Loss$^{512}_G$ \cite{ustinova2016histogram}    & 50.3  & 61.9  & 72.6 & 82.4  \\
        N-Pair-Loss$^{64}_G$ \cite{sohn2016improved}            & 51.0  & 63.3  & 74.3 & 83.2  \\
        Clustering$^{64}_G$ \cite{song2017cvpr}                 & 48.2  & 61.4  & 71.8 & 81.9  \\
        Proxy NCA$^{512}_G$ \cite{movshovitz-attias2017iccv}     & 49.2  & 61.9  & 67.9 & 72.4  \\
        Smart Mining$^{64}_G$ \cite{harwood2017iccv}            & 49.8  & 62.3  & 74.1 & 83.3  \\
        Margin$^{128}_G$ \cite{wu2017iccv}                       & 63.8  & 74.4  & 83.1 & 90.0  \\
        HDC$^{384}_G$ \cite{yuan2016hard}                        & 53.6  & 65.7  & 77.0 & 85.6  \\
        Angular Loss$^{512}_G$ \cite{wang2017iccv}               & 54.7  & 66.3  & 76.0 & 83.9  \\
        HTL$^{128}_G$ \cite{ge2018deep}                          & 57.1  & 68.8  & 78.7 & 86.5  \\
        A-BIER$^{512}_G$ \cite{opitz2018deep}                    & 57.5  & 68.7  & 78.3 & 86.2  \\
        ABE-8$^{512}_G$ \cite{kim2018attention}                  & 60.6  & 71.5  & 80.5 & 87.7  \\
        Proxy NCA$^{512}_{R50}$ \cite{metriclearning-repo}  & \textbf{64.0}  & \textbf{75.4}  & \textbf{84.2} & 90.5  \\
        \hline\Tstrut
        \textbf{\method}$^{512}_{R50}\log$                      & 63.5 & 74.8 & 84.1 & 90.4 \\
        \textbf{\method}$^{512}_{R50}\log\log$                      & \textbf{64.0} & 75.3 & 84.1 & \textbf{90.6} \\
        \hline
    \end{tabular*}
		\caption{Comparison with the state-of-the-art on the
			CUB-200-2011~\cite{cub200} dataset. Our method \method{} is on-par with an
			(unofficial) ResNet50 implementation of Proxy NCA.}
    \label{tbl:sota-cub-200-2011-cars-196}
\end{table}

\begin{table}
    \centering
    \renewcommand{\arraystretch}{1.15}
              \setlength{\tabcolsep}{0pt}
    \begin{tabular*}{\linewidth}{@{\extracolsep\fill}l@{\ \ }|ccccc@{\ }}
    \hline\Tstrut
    $\Rak$                                                                 & 1              & 10             & 20             & 30             & 50             \\
    \hline\Tstrut
    FashionNet$_V$ \cite{liu2016deepfashion}                                   & 53.0           & 73.0           & 76.0           & 77.0           & 80.0           \\
    HDC$^{384}_G$ \cite{yuan2016hard}                                                & 62.1           & 84.9           & 89.0           & 91.2           & 93.1           \\
    DREML$^{48}_{R18}$ \cite{xuan2018deep}                                              & 78.4           & 93.7           & 95.8           & 96.7           & -              \\
    HTL$^{128}_G$ \cite{ge2018deep}                                                  & 80.9           & 94.3           & 95.8           & 97.2           & 97.8           \\
    A-BIER$^{512}_G$ \cite{opitz2018deep}                                            & 83.1           & 95.1           & 96.9           & 97.5           & 98.0           \\
    ABE-8$^{512}_G$ \cite{kim2018attention}                                          & 87.3           & 96.7           & 97.9           & 98.2           & 98.7  \\
    FastAP-Matlab$^{512}_{R50}$ \cite{cakir2019deep} & \textbf{90.9} & \textbf{97.7} & \textbf{98.5} & \textbf{98.8} & \textbf{99.1} \\
    FastAP-Python$^{512}_{R50}$ \cite{cakir2019fastap-repo}  \footnotemark{} \addtocounter{footnote}{-1}                               & 83.8?           & 95.5?           & 96.9?           & 97.5?           & 98.2?           \\
    \hline\Tstrut
    \textbf{\method}$^{512}_{R50}\log$                                                       & 88.1           & 97.0  & 97.9           & 98.4  & 98.8   \\    
        \textbf{\method}$^{512}_{R50}\log\log$                                                       & 86.3           & 96.2  & 97.4           & 97.9  & 98.5  
    \\
    \hline
    \end{tabular*}
		\caption{Comparison with the state-of-the-art methods on the In-shop
			Clothes~\cite{inshop} dataset. \method is on par with an ensemble-method
			ABE-8. Leading performance is achieved with a Matlab implementation of
			FastAP.}
    \label{tbl:sota-clothes}
\end{table}

\paragraph{Results}

For all retrieval results in the tables we add the embedding dimension as a
superscript and the backbone architecture as a subscript. The letters R, G, V
represent ResNet \cite{he2017deep}, GoogLeNet \cite{szegedy2015googlenet}, and
VGG-16 \cite{simonyan2015very}, respectively. We report results for both
\method$^{512}_{R50}\log$ and \method$^{512}_{R50}\log\log$, the main
difference being if the logarithm is applied once or twice to the rank in
Eq.~\eqref{eq:lrec-log}.

On Stanford Online Products we report $\Rak$ for $K \in \{1, 10, 100, 1000\}$
in \tab{tbl:sota-products}. The fact that the dataset contains the highest
number of classes seems to favor \method, as it outperforms all other methods.
Some example retrievals are presented in \fig{fig:online-prod}.

On CUB-200-2011 we report $\Rak$ for $K \in \{1, 2, 4, 8\}$ in
\tab{tbl:sota-cub-200-2011-cars-196}. For fairness, we include the performance
of Proxy NCA with a ResNet50 \cite{he2016deep} backbone even though the results
are only reported in an online implementation \cite{metriclearning-repo}. With
this implementation Proxy NCA and \method are the best-performing methods.

On In-shop Clothes we report $\Rak$ for value of $K \in \{1, 10, 20, 30, 50\}$
in \tab{tbl:sota-clothes}. The best-performing method is probably FastAP, even
though the situation regarding reproducibility is
puzzling\footnotemark{}.  \method matches the performance of
ABE-8~\cite{kim2018attention}, a complex ensemble method.

We followed the reporting strategy of \citep{kim2018attention} by evaluating on
the test set in regular training intervals and reporting performance at a
time-point that maximizes $\mathit{R@1}$.

\footnotetext{FastAP public code \citep{cakir2019fastap-repo} offers Matlab and
PyTorch implementations. Confusingly, the two implementations give very
different results. We contacted the authors but neither we nor they were able
to identify the source of this discrepancy in two seemingly identical
implementations. We report both numbers.}

\subsection{Object detection}
\label{subsec:object_detection}

\noindent
We follow a common protocol for testing new components by using Faster R-CNN
\cite{ren2015faster}, the most commonly used model in object detection, with
standard hyperparameters for all our experiment. We compare against baselines
from the highly optimized mmdetection toolbox \cite{mmdetection} and only
exchange the cross-entropy loss of the classifier with a weighted combination
of $\lmap$ and $\lapc$.

\paragraph{Datasets and evaluation} All experiments are performed on the widely
used Pascal VOC dataset \cite{everingham2010pascal}. We train our models on the
\pascal \vocvii and \vocxii \texttt{trainval} sets and test them on the \vocvii
\texttt{test} set.  Performance is measured in $\ap^{50}$ which is $\ap$
computed for bounding boxes with at least $50\%$ intersection-over-union
overlap with any of the ground truth bounding boxes.

\paragraph{Parameters}
The model was trained for 12 epochs on a single GPU with a batch-size of 8. The
initial learning rate 0.1 is reduced by a factor of 10 after 9 epochs.  For the
$\lap$ loss, we use $\tau=7$, $\alpha=0.15$, and
$\lambda=0.5$. The losses $\lmap$ and $\lapc$ are weighted in the $2:1$ ratio.

\paragraph{Results}
We evaluate Faster R-CNN trained on \vocvii and \vocboth with three different
backbones (ResNet50, ResNet101, and ResNeXt101 32x4d \cite{he2016deep,
xie2017aggregated}). Training with our $\ap$ loss gives a consistent
improvement (see \tab{table:object_detection_voc}) and pushes the standard
Faster R-CNN very close to state-of-the-art values ($\approx 84.1$) achieved by
significantly more complex architectures \cite{zhang2018single,
kim2018parallel}.

\begin{table}
	\centering
	\renewcommand{\arraystretch}{1.15}
	\begin{tabular*}{\linewidth}{@{\hskip1ex\extracolsep\fill}l|c@{\hskip1ex}c|c@{\hskip1ex}c@{\hskip1ex}}
		\hline\Tstrut
		Method & Backbone & Training  & CE & \method \\
		\hline\Tstrut
		Faster R-CNN & ResNet50 & \texttt{07} & 74.2 & \textbf{75.7}\\
		Faster R-CNN & ResNet50 & \texttt{07+12} & 80.4 & \textbf{81.4}\\
		Faster R-CNN  & ResNet101 & \texttt{07+12} & 82.4 & \textbf{82.9}\\
		Faster R-CNN & X101 32$\times$4d & \texttt{07+12} & 83.2 & \textbf{83.6}\\
		\hline
	\end{tabular*}
	\caption{Object detection performance on the \pascal \vocvii \texttt{test}
		set measured in $\ap^{50}$. Backbone X stands for ResNeXt and CE for cross
		entropy loss.}
	\label{table:object_detection_voc}
\end{table}

\subsection{Speed}

\noindent
Since \method can be implemented using sorting functions it is very fast to
compute (see \tab{table:object_detection_timing}) and can be used on very long
sequences. Computing $\ap$ loss for sequences with 320k elements as in the
object detection experiments takes less than 5\,ms for the forward/backward
pass.  This is $<0.5\%$ of the overall computation time on a batch.

\begin{table}
  \centering
  \renewcommand{\arraystretch}{1.15}
  \begin{tabular*}{\linewidth}{@{\hskip1ex\extracolsep\fill}l|cccc}
    \hline\Tstrut
    Length & 100k & 1M & 10M  & 100M \\
    \hline\Tstrut
    CPU & 33\,ms & 331\,ms & 3.86\,s  & 36.4\,s   \\
    GPU & 1.3\,ms & 7\,ms  & 61\,ms & 0.62\,s  \\
    \hline
  \end{tabular*}
	\caption{Processing time of Average Precision (using plain \textsc{Pytorch}
		implementation) depending on sequence length for forward/backward computation
		on a single Tesla V100 GPU and 1 Xeon Gold CPU core at 2.2GHz.}
  \label{table:object_detection_timing}
\end{table}

\begin{table}
	\renewcommand{\arraystretch}{1.15}
	\centering
	\begin{tabular*}{\linewidth}{@{\ \extracolsep\fill}l|cccc@{\ }}
		\hline\Tstrut
		$\mathit{R@1}$ & CUB200 & In-shop & Online Prod. \\
		\hline\Tstrut
		Full \method & 64.0 & 88.1 & 78.6 \\
		No batch memory  & 62.5 & 87.0 & 72.4 \\
		No margin & 63.2 & x & x\\
		\hline
	\end{tabular*}
	\caption{Ablation experiments for margin(\sec{sec:margin}) and batch
		memory (\sec{sec:memory}) in retrieval on the CUB200,  In-shop and
		Stanford Online Products datasets.}
	\label{table:ablation-studies}
\end{table}

\subsection{Ablation studies}

\noindent
We verify the validity of our loss design in multiple ablation studies.
\Tab{table:ablation-studies} shows the relevance of margin and batch memory for
the retrieval task. In fact, some of the runs without a margin diverged.
The importance of margin is also shown for the $\map$ loss in
\tab{table:object_detection_ablations}. Moreover, we can see that the
hyperparameter $\lambda$ of the scheme \cite{blackbox-diff} does not need
precise tuning. Values of $\lambda$ that are within a factor 5 of the selected
$\lambda=0.5$ still outperform the baseline.

\begin{table}
	\centering
	\renewcommand{\arraystretch}{1.15}
	\begin{tabular*}{\linewidth}{@{\hskip1ex\extracolsep\fill}l|ccc|c@{\hskip1ex}}
		\hline\Tstrut
		Method  & \method & $\lambda$ & margin& $\ap^{50}$ \\
		\hline\Tstrut
		Faster R-CNN & & & & 74.2\\
		Faster R-CNN & \checkmark & 0.5 & & 74.6\\
		Faster R-CNN & \checkmark & 0.1 & \checkmark & 75.2\\
		Faster R-CNN & \checkmark & 0.5 & \checkmark & \textbf{75.7}\\
		Faster R-CNN & \checkmark & 2.5 & \checkmark & 74.3\\
		\hline
	\end{tabular*}
	\caption{Ablation for \method on the object detection task.}
	\label{table:object_detection_ablations}
\end{table}

\section{Discussion}

\noindent
The proposed method \method{} is singled out by its conceptual purity in
directly optimizing for the desired metric while being simple, flexible, and
computationally efficient.  Driven only by basic loss-design principles and
without serious engineering efforts, it can compete with state-of-the-art
methods on image retrieval and consistently improve near-state-of-the-art
object detectors.  Exciting opportunities for future work lie in utilizing the
ability to efficiently optimize ranking-metrics of sequences with millions of
elements.

\newpage

{\small
\bibliographystyle{abbrvnat}
%\bibliography{arxiv_camera_ready}

\begin{thebibliography}{71}
\providecommand{\natexlab}[1]{#1}
\providecommand{\url}[1]{\texttt{#1}}
\expandafter\ifx\csname urlstyle\endcsname\relax
  \providecommand{\doi}[1]{doi: #1}\else
  \providecommand{\doi}{doi: \begingroup \urlstyle{rm}\Url}\fi

\bibitem[Bartell et~al.(1994)Bartell, Cottrell, and Belew]{Bartell:1994}
B.~T. Bartell, G.~W. Cottrell, and R.~K. Belew.
\newblock Automatic combination of multiple ranked retrieval systems.
\newblock In \emph{ACM Conference on Research and Development in Information
  Retrieval}, SIGIR'94, pages 173--181. Springer, 1994.

\bibitem[Bellet et~al.(2013)Bellet, Habrard, and Sebban]{bellet2013survey}
A.~Bellet, A.~Habrard, and M.~Sebban.
\newblock A survey on metric learning for feature vectors and structured data.
\newblock \emph{arXiv preprint arXiv:1306.6709}, 2013.

\bibitem[Bromley et~al.(1994)Bromley, Guyon, LeCun, S{\"a}ckinger, and
  Shah]{bromley1994signature}
J.~Bromley, I.~Guyon, Y.~LeCun, E.~S{\"a}ckinger, and R.~Shah.
\newblock Signature verification using a ``siamese'' time delay neural network.
\newblock In \emph{Advances in Neural Information Processing Systems}, NIPS'94,
  pages 737--744, 1994.

\bibitem[Cakir et~al.(2019{\natexlab{a}})Cakir, He, Xia, Kulis, and
  Sclaroff]{cakir2019deep}
F.~Cakir, K.~He, X.~Xia, B.~Kulis, and S.~Sclaroff.
\newblock Deep metric learning to rank.
\newblock In \emph{IEEE Conference on Computer Vision and Pattern Recognition},
  CVPR'19, pages 1861--1870, 2019{\natexlab{a}}.

\bibitem[Cakir et~al.(2019{\natexlab{b}})Cakir, He, Xia, Kulis, and
  Sclaroff]{cakir2019fastap-repo}
F.~Cakir, K.~He, X.~Xia, B.~Kulis, and S.~Sclaroff.
\newblock {Deep Metric Learning to Rank}.
\newblock \url{https://github.com/kunhe/FastAP-metric-learning},
  2019{\natexlab{b}}.
\newblock Commit: 7ca48aa.

\bibitem[Chakrabarti et~al.(2008)Chakrabarti, Khanna, Sawant, and
  Bhattacharyya]{chakrabarti2008structured}
S.~Chakrabarti, R.~Khanna, U.~Sawant, and C.~Bhattacharyya.
\newblock Structured learning for non-smooth ranking losses.
\newblock In \emph{KDD}, 2008.

\bibitem[Chen et~al.(2019{\natexlab{a}})Chen, Li, Lin, See, Wang, Duan, Chen,
  He, and Zou]{chen2019towards}
K.~Chen, J.~Li, W.~Lin, J.~See, J.~Wang, L.~Duan, Z.~Chen, C.~He, and J.~Zou.
\newblock Towards accurate one-stage object detection with ap-loss.
\newblock In \emph{IEEE Conference on Computer Vision and Pattern Recognition},
  CVPR'19, pages 5119--5127, 2019{\natexlab{a}}.

\bibitem[Chen et~al.(2019{\natexlab{b}})Chen, Wang, Pang, Cao, Xiong, Li, Sun,
  Feng, Liu, Xu, Zhang, Cheng, Zhu, Cheng, Zhao, Li, Lu, Zhu, Wu, Dai, Wang,
  Shi, Ouyang, Loy, and Lin]{mmdetection}
K.~Chen, J.~Wang, J.~Pang, Y.~Cao, Y.~Xiong, X.~Li, S.~Sun, W.~Feng, Z.~Liu,
  J.~Xu, Z.~Zhang, D.~Cheng, C.~Zhu, T.~Cheng, Q.~Zhao, B.~Li, X.~Lu, R.~Zhu,
  Y.~Wu, J.~Dai, J.~Wang, J.~Shi, W.~Ouyang, C.~C. Loy, and D.~Lin.
\newblock {MMDetection}: Open {MMLab} detection toolbox and benchmark.
\newblock \emph{arXiv preprint arXiv:1906.07155}, 2019{\natexlab{b}}.
\newblock Commit: 9d767a03c0ee60081fd8a2d2a200e530bebef8eb.

\bibitem[Cohendet et~al.(2018)Cohendet, Demarty, Duong, Sj{\"o}berg, Ionescu, ,
  and Do]{cohendet2018media}
R.~Cohendet, C.-H. Demarty, N.~Duong, M.~Sj{\"o}berg, B.~Ionescu, , and T.-T.
  Do.
\newblock {MediaEval} 2018: Predicting media memorability.
\newblock \emph{arXiv:1807.01052}, 2018.

\bibitem[Engilberge et~al.(2019)Engilberge, Chevallier, P{\'e}rez, and
  Cord]{engilberge2019sodeep}
M.~Engilberge, L.~Chevallier, P.~P{\'e}rez, and M.~Cord.
\newblock Sodeep: a sorting deep net to learn ranking loss surrogates.
\newblock In \emph{IEEE Conference on Computer Vision and Pattern Recognition},
  CVPR'19, pages 10792--10801, 2019.

\bibitem[Everingham et~al.(2010)Everingham, Van~Gool, Williams, Winn, and
  Zisserman]{everingham2010pascal}
M.~Everingham, L.~Van~Gool, C.~Williams, J.~Winn, and A.~Zisserman.
\newblock The pascal visual object classes (voc) challenge.
\newblock \emph{International Journal of Computer Vision}, 2010.

\bibitem[Ge(2018)]{ge2018deep}
W.~Ge.
\newblock Deep metric learning with hierarchical triplet loss.
\newblock In \emph{European Conference on Computer Vision}, ECCV'18, pages
  269--285, 2018.

\bibitem[Ghiasi et~al.(2019)Ghiasi, Lin, and Le]{ghiasi2019fpn}
G.~Ghiasi, T.-Y. Lin, and Q.~V. Le.
\newblock Nas-fpn: Learning scalable feature pyramid architecture for object
  detection.
\newblock In \emph{IEEE Conference on Computer Vision and Pattern Recognition},
  CVPR'19, 2019.

\bibitem[Girshick et~al.(2014)Girshick, Donahue, Darrell, and
  Malik]{girshick2014rich}
R.~Girshick, J.~Donahue, T.~Darrell, and J.~Malik.
\newblock Rich feature hierarchies for accurate object detection and semantic
  segmentation.
\newblock In \emph{IEEE Conference on Computer Vision and Pattern Recognition},
  CVPR'14, pages 580--587, 2014.

\bibitem[Goldman et~al.(2019)Goldman, Herzig, Eisenschtat, Goldberger, and
  Hassner]{Goldman_2019_CVPR}
E.~Goldman, R.~Herzig, A.~Eisenschtat, J.~Goldberger, and T.~Hassner.
\newblock Precise detection in densely packed scenes.
\newblock In \emph{The IEEE Conference on Computer Vision and Pattern
  Recognition}, CVPR'19, 2019.

\bibitem[Hardy et~al.(1952)Hardy, Littlewood, and
  P{\'o}lya]{hardy1952inequalities}
G.~H. Hardy, J.~E. Littlewood, and G.~P{\'o}lya.
\newblock \emph{Inequalities}.
\newblock Cambridge University Press, Cambridge, England, 1952.

\bibitem[Harwood et~al.(2017)Harwood, Kumar, Carneiro, Reid, Drummond,
  et~al.]{harwood2017iccv}
B.~Harwood, B.~Kumar, G.~Carneiro, I.~Reid, T.~Drummond, et~al.
\newblock Smart mining for deep metric learning.
\newblock In \emph{IEEE International Conference on Computer Vision}, ICCV'17,
  pages 2821--2829, 2017.

\bibitem[He et~al.(2016)He, Zhang, Ren, and Sun]{he2016deep}
K.~He, X.~Zhang, S.~Ren, and J.~Sun.
\newblock Deep residual learning for image recognition.
\newblock In \emph{IEEE Conference on Computer Vision and Pattern Recognition},
  CVPR'18, pages 770--778, 2016.

\bibitem[He et~al.(2017{\natexlab{a}})He, Gkioxari, Doll{\'a}r, and
  Girshick]{he2017mask}
K.~He, G.~Gkioxari, P.~Doll{\'a}r, and R.~Girshick.
\newblock Mask {R-CNN}.
\newblock In \emph{IEEE International Conference on Computer Vision}, ICCV'17,
  pages 2980--2988, 2017{\natexlab{a}}.

\bibitem[He et~al.(2018{\natexlab{a}})He, Cakir, Adel~Bargal, and
  Sclaroff]{he2018hashing}
K.~He, F.~Cakir, S.~Adel~Bargal, and S.~Sclaroff.
\newblock Hashing as tie-aware learning to rank.
\newblock In \emph{IEEE Conference on Computer Vision and Pattern Recognition},
  CVPR'18, pages 4023--4032, 2018{\natexlab{a}}.

\bibitem[He et~al.(2018{\natexlab{b}})He, Lu, and Sclaroff]{he2018local}
K.~He, Y.~Lu, and S.~Sclaroff.
\newblock Local descriptors optimized for average precision.
\newblock In \emph{IEEE Conference on Computer Vision and Pattern Recognition},
  CVPR'18, pages 596--605, 2018{\natexlab{b}}.

\bibitem[He et~al.(2017{\natexlab{b}})He, Zhang, Yin, and Liu]{he2017deep}
W.~He, X.-Y. Zhang, F.~Yin, and C.-L. Liu.
\newblock Deep direct regression for multi-oriented scene text detection.
\newblock In \emph{IEEE International Conference on Computer Vision}, ICCV'17,
  2017{\natexlab{b}}.

\bibitem[Henderson and Ferrari(2016)]{henderson2016end}
P.~Henderson and V.~Ferrari.
\newblock End-to-end training of object class detectors for mean average
  precision.
\newblock In \emph{Asian Conference on Computer Vision}, pages 198--213.
  Springer, 2016.

\bibitem[Hoffer and Ailon(2015)]{hoffer2015deep}
E.~Hoffer and N.~Ailon.
\newblock Deep metric learning using triplet network.
\newblock In \emph{International Workshop on Similarity-Based Pattern
  Recognition}, pages 84--92. Springer, 2015.

\bibitem[Huang et~al.(2017)Huang, Rathod, Sun, Zhu, Korattikara, Fathi,
  Fischer, Wojna, Song, Guadarrama, and Murphy]{tf-object-detection-api}
J.~Huang, V.~Rathod, C.~Sun, M.~Zhu, A.~Korattikara, A.~Fathi, I.~Fischer,
  Z.~Wojna, Y.~Song, S.~Guadarrama, and K.~Murphy.
\newblock {Tensorflow Object Detection API}.
\newblock
  \url{https://github.com/tensorflow/models/tree/master/research/object_detection},
  2017.
\newblock Commit: 0ba83cf.

\bibitem[Kim et~al.(2018{\natexlab{a}})Kim, Kook, Sun, Kang, and
  Ko]{kim2018parallel}
S.-W. Kim, H.-K. Kook, J.-Y. Sun, M.-C. Kang, and S.-J. Ko.
\newblock Parallel feature pyramid network for object detection.
\newblock In \emph{European Conference on Computer Vision}, ECCV'18, pages
  230--256, 2018{\natexlab{a}}.

\bibitem[Kim et~al.(2018{\natexlab{b}})Kim, Goyal, Chawla, Lee, and
  Kwon]{kim2018attention}
W.~Kim, B.~Goyal, K.~Chawla, J.~Lee, and K.~Kwon.
\newblock Attention-based ensemble for deep metric learning.
\newblock In \emph{European Conference on Computer Vision}, ECCV'18, pages
  736--751, 2018{\natexlab{b}}.

\bibitem[Kingma and Ba(2014)]{kingma2014adam}
D.~P. Kingma and J.~Ba.
\newblock {Adam}: A method for stochastic optimization.
\newblock In \emph{International Conference on Learning Representations},
  ICLR'14, 2014.

\bibitem[Koch et~al.(2015)Koch, Zemel, and Salakhutdinov]{koch2015siamese}
G.~Koch, R.~Zemel, and R.~Salakhutdinov.
\newblock Siamese neural networks for one-shot image recognition.
\newblock In \emph{ICML deep learning workshop}, volume~2, 2015.

\bibitem[Kulis et~al.(2013)]{kulis2013metric}
B.~Kulis et~al.
\newblock Metric learning: A survey.
\newblock \emph{Foundations and Trends in Machine Learning}, 5\penalty0
  (4):\penalty0 287--364, 2013.

\bibitem[Law et~al.(2013)Law, Thome, and Cord]{law2013quadruplet}
M.~T. Law, N.~Thome, and M.~Cord.
\newblock Quadruplet-wise image similarity learning.
\newblock In \emph{IEEE International Conference on Computer Vision}, ICCV'13,
  pages 249--256, 2013.

\bibitem[Lin et~al.(2018)Lin, Goyal, Girshick, He, and
  Doll{\'a}r]{lin2018focal}
T.-Y. Lin, P.~Goyal, R.~Girshick, K.~He, and P.~Doll{\'a}r.
\newblock Focal loss for dense object detection.
\newblock \emph{IEEE Transactions on Pattern Analysis and Machine
  Intelligence}, 2018.

\bibitem[Lin et~al.(2002)Lin, Lee, and Wahba]{Lin2002}
Y.~Lin, Y.~Lee, and G.~Wahba.
\newblock Support vector machines for classification in nonstandard situations.
\newblock \emph{Machine Learning}, 46\penalty0 (1):\penalty0 191--202, 2002.
\newblock \doi{10.1023/A:1012406528296}.

\bibitem[Liu et~al.(2016{\natexlab{a}})Liu, Anguelov, Erhan, Szegedy, Reed, Fu,
  and Berg]{liu2015ssd}
W.~Liu, D.~Anguelov, D.~Erhan, C.~Szegedy, S.~Reed, C.-Y. Fu, and A.~C. Berg.
\newblock Ssd: Single shot multibox detector.
\newblock In \emph{ECCV}, pages 21--37. Springer, 2016{\natexlab{a}}.

\bibitem[Liu et~al.(2016{\natexlab{b}})Liu, Luo, Qiu, Wang, and Tang]{inshop}
Z.~Liu, P.~Luo, S.~Qiu, X.~Wang, and X.~Tang.
\newblock Deepfashion: Powering robust clothes recognition and retrieval with
  rich annotations.
\newblock In \emph{IEEE Conference on Computer Vision and Pattern Recognition},
  CVPR'16, 2016{\natexlab{b}}.

\bibitem[Liu et~al.(2016{\natexlab{c}})Liu, Luo, Qiu, Wang, and
  Tang]{liu2016deepfashion}
Z.~Liu, P.~Luo, S.~Qiu, X.~Wang, and X.~Tang.
\newblock Deepfashion: Powering robust clothes recognition and retrieval with
  rich annotations.
\newblock In \emph{IEEE Conference on Computer Vision and Pattern Recognition},
  CVPR'16, pages 1096--1104, 2016{\natexlab{c}}.

\bibitem[Massa and Girshick(2018)]{mrcnn-benchmark}
F.~Massa and R.~Girshick.
\newblock {maskrcnn-benchmark: Fast, modular reference implementation of
  Instance Segmentation and Object Detection algorithms in PyTorch}.
\newblock \url{https://github.com/facebookresearch/maskrcnn-benchmark}, 2018.
\newblock Commit: f027259.

\bibitem[McFee and Lanckriet(2010)]{mcfee2010metric}
B.~McFee and G.~R. Lanckriet.
\newblock Metric learning to rank.
\newblock In \emph{International Conference on Machine Learning}, ICML'10,
  pages 775--782, 2010.

\bibitem[Mohapatra et~al.(2014)Mohapatra, Jawahar, and
  Kumar]{mohapatra2014efficient}
P.~Mohapatra, C.~Jawahar, and M.~P. Kumar.
\newblock Efficient optimization for average precision svm.
\newblock In \emph{Advances in Neural Information Processing Systems}, pages
  2312--2320, 2014.

\bibitem[Mohapatra et~al.(2018)Mohapatra, Rolinek, Jawahar, Kolmogorov, and
  Pawan~Kumar]{mohapatra2018efficient}
P.~Mohapatra, M.~Rolinek, C.~Jawahar, V.~Kolmogorov, and M.~Pawan~Kumar.
\newblock Efficient optimization for rank-based loss functions.
\newblock In \emph{IEEE Conference on Computer Vision and Pattern Recognition},
  CVPR'18, pages 3693--3701, 2018.

\bibitem[Movshovitz-Attias et~al.(2017)Movshovitz-Attias, Toshev, Leung, Ioffe,
  and Singh]{movshovitz-attias2017iccv}
Y.~Movshovitz-Attias, A.~Toshev, T.~K. Leung, S.~Ioffe, and S.~Singh.
\newblock No fuss distance metric learning using proxies.
\newblock In \emph{IEEE International Conference on Computer Vision}, ICCV'17,
  pages 360--368, 2017.

\bibitem[Oh~Song et~al.(2016)Oh~Song, Xiang, Jegelka, and Savarese]{oh2016deep}
H.~Oh~Song, Y.~Xiang, S.~Jegelka, and S.~Savarese.
\newblock Deep metric learning via lifted structured feature embedding.
\newblock In \emph{IEEE Conference on Computer Vision and Pattern Recognition},
  CVPR'16, pages 4004--4012, 2016.

\bibitem[Oh~Song et~al.(2017)Oh~Song, Jegelka, Rathod, and
  Murphy]{song2017cvpr}
H.~Oh~Song, S.~Jegelka, V.~Rathod, and K.~Murphy.
\newblock Deep metric learning via facility location.
\newblock In \emph{IEEE Conference on Computer Vision and Pattern Recognition},
  CVPR'17, pages 5382--5390, 2017.

\bibitem[Opitz et~al.(2018)Opitz, Waltner, Possegger, and
  Bischof]{opitz2018deep}
M.~Opitz, G.~Waltner, H.~Possegger, and H.~Bischof.
\newblock Deep metric learning with {BIER}: Boosting independent embeddings
  robustly.
\newblock \emph{IEEE Transactions on Pattern Analysis and Machine
  Intelligence}, 2018.

\bibitem[Rao et~al.(2018)Rao, Lin, Lu, and Zhou]{Rao2018LearningGO}
Y.~Rao, D.~Lin, J.~Lu, and J.~Zhou.
\newblock Learning globally optimized object detector via policy gradient.
\newblock In \emph{IEEE Conference on Computer Vision and Pattern Recognition},
  CVPR'18, pages 6190--6198, 2018.

\bibitem[Redmon and Farhadi(2018)]{Redmon2018YOLOv3}
J.~Redmon and A.~Farhadi.
\newblock Yolov3: An incremental improvement.
\newblock \emph{arXiv preprint arXiv:1804.02767}, 2018.

\bibitem[Redmon et~al.(2016)Redmon, Divvala, Girshick, and
  Farhadi]{Redmon_2016_CVPR}
J.~Redmon, S.~Divvala, R.~Girshick, and A.~Farhadi.
\newblock You only look once: Unified, real-time object detection.
\newblock In \emph{IEEE Conference on Computer Vision and Pattern Recognition},
  CVPR'16, 2016.

\bibitem[Ren et~al.(2015)Ren, He, Girshick, and Sun]{ren2015faster}
S.~Ren, K.~He, R.~Girshick, and J.~Sun.
\newblock Faster r-cnn: Towards real-time object detection with region proposal
  networks.
\newblock In \emph{Advances in Neural Information Processing Systems}, NIPS'15,
  pages 91--99, 2015.

\bibitem[Revaud et~al.(2019)Revaud, Almazan, de~Rezende, and
  de~Souza]{revaud2019learning}
J.~Revaud, J.~Almazan, R.~S. de~Rezende, and C.~R. de~Souza.
\newblock Learning with {A}verage {P}recision: Training image retrieval with a
  listwise loss.
\newblock In \emph{IEEE International Conference on Computer Vision}, ICCV'19,
  2019.

\bibitem[Rezatofighi et~al.(2019)Rezatofighi, Tsoi, Gwak, Sadeghian, Reid, and
  Savarese]{rezatofighi2019generalized}
H.~Rezatofighi, N.~Tsoi, J.~Gwak, A.~Sadeghian, I.~Reid, and S.~Savarese.
\newblock Generalized intersection over union: A metric and a loss for bounding
  box regression.
\newblock In \emph{IEEE Conference on Computer Vision and Pattern Recognition},
  CVPR'19, 2019.

\bibitem[Roth and Brattoli(2019)]{metriclearning-repo}
K.~Roth and B.~Brattoli.
\newblock Easily extendable basic deep metric learning pipeline.
\newblock \url{https://github.com/Confusezius/Deep-Metric-Learning-Baselines},
  2019.
\newblock Commit: 59d48f9.

\bibitem[Schroff et~al.(2015)Schroff, Kalenichenko, and
  Philbin]{schroff2015facenet}
F.~Schroff, D.~Kalenichenko, and J.~Philbin.
\newblock Facenet: A unified embedding for face recognition and clustering.
\newblock In \emph{IEEE Conference on Computer Vision and Pattern Recognition},
  CVPR'15, pages 815--823, 2015.

\bibitem[Simonyan and Zisserman(2015)]{simonyan2015very}
K.~Simonyan and A.~Zisserman.
\newblock Very deep convolutional networks for large-scale image recognition.
\newblock In \emph{International Conference on Learning Representations},
  ICLR'15, 2015.

\bibitem[Sohn(2016)]{sohn2016improved}
K.~Sohn.
\newblock Improved deep metric learning with multi-class n-pair loss objective.
\newblock In \emph{Advances in Neural Information Processing Systems}, NIPS'16,
  pages 1857--1865, 2016.

\bibitem[Song et~al.(2016{\natexlab{a}})Song, Xiang, Jegelka, and
  Savarese]{stanford-products}
H.~O. Song, Y.~Xiang, S.~Jegelka, and S.~Savarese.
\newblock Deep metric learning via lifted structured feature embedding.
\newblock In \emph{IEEE Conference on Computer Vision and Pattern Recognition},
  CVPR'16, 2016{\natexlab{a}}.

\bibitem[Song et~al.(2016{\natexlab{b}})Song, Schwing, Urtasun,
  et~al.]{song2016training}
Y.~Song, A.~Schwing, R.~Urtasun, et~al.
\newblock Training deep neural networks via direct loss minimization.
\newblock In \emph{International Conference on Machine Learning}, ICML'16,
  pages 2169--2177, 2016{\natexlab{b}}.

\bibitem[Szegedy et~al.(2015)Szegedy, Liu, Jia, Sermanet, Reed, Anguelov,
  Erhan, Vanhoucke, and Rabinovich]{szegedy2015googlenet}
C.~Szegedy, W.~Liu, Y.~Jia, P.~Sermanet, S.~Reed, D.~Anguelov, D.~Erhan,
  V.~Vanhoucke, and A.~Rabinovich.
\newblock Going deeper with convolutions.
\newblock In \emph{IEEE Conference on Computer Vision and Pattern Recognition},
  CVPR'15, 2015.

\bibitem[Taylor et~al.(2008)Taylor, Guiver, Robertson, and
  Minka]{taylor2008softrank}
M.~Taylor, J.~Guiver, S.~Robertson, and T.~Minka.
\newblock Softrank: optimizing non-smooth rank metrics.
\newblock In \emph{2008 International Conference on Web Search and Data
  Mining}, pages 77--86. ACM, 2008.

\bibitem[Ustinova and Lempitsky(2016)]{ustinova2016histogram}
E.~Ustinova and V.~Lempitsky.
\newblock Learning deep embeddings with histogram loss.
\newblock In \emph{Advances in Neural Information Processing Systems}, NIPS'16,
  pages 4170--4178, 2016.

\bibitem[Vlastelica et~al.(2020)Vlastelica, Paulus, Musil, Martius, and
  Rol{\'\i}nek]{blackbox-diff}
M.~Vlastelica, A.~Paulus, V.~Musil, G.~Martius, and M.~Rol{\'\i}nek.
\newblock Differentiation of blackbox combinatorial solvers.
\newblock In \emph{International Conference on Learning Representations},
  ICLR'20, 2020.

\bibitem[Wang et~al.(2017)Wang, Zhou, Wen, Liu, and Lin]{wang2017iccv}
J.~Wang, F.~Zhou, S.~Wen, X.~Liu, and Y.~Lin.
\newblock Deep metric learning with angular loss.
\newblock In \emph{IEEE International Conference on Computer Vision}, ICCV'17,
  pages 2593--2601, 2017.

\bibitem[Welinder et~al.(2010)Welinder, Branson, Mita, Wah, Schroff, Belongie,
  and Perona]{cub200}
P.~Welinder, S.~Branson, T.~Mita, C.~Wah, F.~Schroff, S.~Belongie, and
  P.~Perona.
\newblock {Caltech-UCSD Birds 200}.
\newblock Technical Report CNS-TR-2010-001, California Institute of Technology,
  2010.

\bibitem[Wu et~al.(2017)Wu, Manmatha, Smola, and Krahenbuhl]{wu2017iccv}
C.-Y. Wu, R.~Manmatha, A.~J. Smola, and P.~Krahenbuhl.
\newblock Sampling matters in deep embedding learning.
\newblock In \emph{IEEE International Conference on Computer Vision}, ICCV'17,
  pages 2840--2848, 2017.

\bibitem[Wu et~al.(2019)Wu, Kirillov, Massa, Lo, and Girshick]{detectron2}
Y.~Wu, A.~Kirillov, F.~Massa, W.-Y. Lo, and R.~Girshick.
\newblock Detectron2.
\newblock \url{https://github.com/facebookresearch/detectron2}, 2019.
\newblock Commit: dd5926a.

\bibitem[Xie et~al.(2017)Xie, Girshick, Doll{\'a}r, Tu, and
  He]{xie2017aggregated}
S.~Xie, R.~Girshick, P.~Doll{\'a}r, Z.~Tu, and K.~He.
\newblock Aggregated residual transformations for deep neural networks.
\newblock In \emph{IEEE Conference on Computer Vision and Pattern Recognition},
  CVPR'17, pages 5987--5995, 2017.

\bibitem[Xuan et~al.(2018)Xuan, Souvenir, and Pless]{xuan2018deep}
H.~Xuan, R.~Souvenir, and R.~Pless.
\newblock Deep randomized ensembles for metric learning.
\newblock In \emph{European Conference on Computer Vision}, ECCV'18, pages
  723--734, 2018.

\bibitem[Yuan et~al.(2017)Yuan, Yang, and Zhang]{yuan2016hard}
Y.~Yuan, K.~Yang, and C.~Zhang.
\newblock Hard-aware deeply cascaded embedding.
\newblock In \emph{IEEE International Conference on Computer Vision}, ICCV'17,
  pages 814--823, 2017.

\bibitem[Yue et~al.(2007)Yue, Finley, Radlinski, and Joachims]{yue2007support}
Y.~Yue, T.~Finley, F.~Radlinski, and T.~Joachims.
\newblock A support vector method for optimizing average precision.
\newblock In \emph{ACM SIGIR Conference on Research and Development in
  Information Retrieval}, pages 271--278. ACM, 2007.

\bibitem[Zhang et~al.(2018)Zhang, Wen, Bian, Lei, and Li]{zhang2018single}
S.~Zhang, L.~Wen, X.~Bian, Z.~Lei, and S.~Z. Li.
\newblock Single-shot refinement neural network for object detection.
\newblock In \emph{IEEE Conference on Computer Vision and Pattern Recognition},
  CVPR'18, pages 4203--4212, 2018.

\bibitem[Zhao et~al.(2018)Zhao, Jin, Qi, Lu, and Hua]{zhao2018adversarial}
Y.~Zhao, Z.~Jin, G.-j. Qi, H.~Lu, and X.-s. Hua.
\newblock An adversarial approach to hard triplet generation.
\newblock In \emph{European Conference on Computer Vision}, ECCV'18, pages
  501--517, 2018.

\bibitem[Zoph et~al.(2018)Zoph, Vasudevan, Shlens, and Le]{zoph2018learning}
B.~Zoph, V.~Vasudevan, J.~Shlens, and Q.~V. Le.
\newblock Learning transferable architectures for scalable image recognition.
\newblock In \emph{IEEE Conference on Computer Vision and Pattern Recognition},
  CVPR'18, 2018.

\end{thebibliography}

\flushcolsend
}

\newpage
\appendix

\section{Parameters of retrieval experiments}

\noindent
In all experiments we used the ADAM optimizer with a weight decay value of
$4\times10^{-4}$ and batch size 128.  All experiments ran at most 80 epochs
with a learning rate drop by $70\%$ after 35 epochs and a batch memory of
length 3.  We used higher learning rates for the embedding layer as specified
by defaults in \citet{cakir2019fastap-repo}.

We used a super-label batch preparation strategy in which we sample a
consecutive batches for the same super-label pair, as specified by
\citet{cakir2019fastap-repo}.  For the In-shop Clothes dataset we used 4
batches per pair of super-labels and 8 samples per class within a batch.  In
the Online Products dataset we used 10 batches per pair of super-labels along
with 4 samples per class within a batch.  For CUB200, there are no super-labels
and we just sample 4 examples per classes within a batch. These values again
follow \citet{cakir2019fastap-repo}.  The remaining settings are in
\Tab{tbl:recall-parameters}.

\begin{table}[h]
	\centering
	\renewcommand{\arraystretch}{1.15}
	\setlength{\tabcolsep}{0pt}
	\begin{tabular*}{\linewidth}{@{\extracolsep\fill}l@{\ \ }|ccc@{\ }}
		& Online Products & In-shop & CUB200 \\
		\hline\Tstrut
		$lr$ & $3\times 10^{-6}$ &  $10^{-5}$ & $5 \times 10^{-6}$\\
		margin & 0.02 & 0.05 & 0.02\\
		$\lambda$ & 4 & 0.2 & 0.2\\
		\hline
	\end{tabular*}
	\caption{Hyperparameter values for retrieval experiments.}
	\label{tbl:recall-parameters}
\end{table}

\section{Proofs}

\begin{lemma} \label{lem:coarea}
Let $\{w_k\}$ be a sequence of nonnegative weights and let $r_1,\dots,r_n$ be
positive integers.  Then
\begin{equation} \label{eq:coarea}
	\sum_{k=1}^\infty w_k |\{i:r_i\ge k\}|
		= \sum_{i=1}^n W(r_i),
\end{equation}
where
\begin{equation} \label{eq:W}
	W(k) = \sum_{i=1}^k w_i
		\quad\text{for $k\in\N$}.
\end{equation}
\end{lemma}

Note that the sum on the left hand-side of \eqref{eq:coarea} is finite.

\begin{prop} \label{prop:recall-loss}
Let $w_K$ be nonnegative weights for $K\in\N$ and assume that $\lrec$ is given
by
\begin{equation} \label{eq:lrec-single-2}
	\lrec(\y,\y^*)
		= \sum_{K=1}^\infty w_K \lossk(\y,\y^*).
\end{equation}
Then
\begin{equation} \label{eq:lrec-closed}
	\lrec(\y,\y^*)
		= \frac{1}{|\rel(\y^*)|}
			\sum_{i\in\rel(\y^*)} W(r_i),
\end{equation}
where $W$ is as in \eqref{eq:W}.
\end{prop}

\begin{proof}
Taking the complement of the set $\rel(\y^*)$ in the definition of $\lossk$, we
get
\begin{equation} \label{eq:lak}
	\lossk(\y,\y^*)
		= \frac{|\{i\in\rel(\y^*):r_i\ge K\}|}{|\rel(\y^*)|},
\end{equation}
whence \eqref{eq:lrec-single-2} reads as
\begin{equation*}
	\lrec(\y,\y^*)
		= \frac{1}{|\rel(\y^*)|}
			\sum_{k=1}^\infty w_K |\{i:r_i\ge K\}|.
\end{equation*}
Equation \eqref{eq:lrec-closed} then follows by Lemma~\ref{lem:coarea}.
\end{proof}

\begin{proof}[proof of Lemma~\ref{lem:coarea}]
Observe that $w_k = W(k)-W(k-1)$ and $W(0)=0$. Then
\begin{align*}
	\sum_{i=1}^n W(r_i)
		& = \sum_{k=1}^\infty W(k) |\{i:r_i=k\}|
			\\
		& = \sum_{k=1}^\infty W(k) \bigl|\{i:r_i\ge k\}\setminus\{i:r_i\ge k+1\}\bigr|
			\\
		& = \sum_{k=1}^\infty W(k) |\{i:r_i\ge k\}|
			\\
		& \quad - \sum_{k=1}^\infty W(k-1) |\{i:r_i\ge k\}|
			\\
		& = \sum_{k=1}^\infty \bigl(W(k) - W(k-1)\bigr) |\{i:r_i\ge k\}|
			\\
		& = \sum_{k=1}^\infty w_k |\{i:r_i\ge k\}|
\end{align*}
and \eqref{eq:coarea} follows.
\end{proof}

\begin{proof}[Proof of \eqref{eq:lrec-log}]
Let us set $w_k = \log(1+1/k)$ for $k\in\N$.  Then from Taylor's expansion of
$\log$ we have the desired $w_k \approx \frac1k$ and
\begin{align*}
	W(k)
		& = \sum_{i=1}^k \log\left( 1+\frac{1}{i} \right)
			\\
		& = \log\left(\prod_{i=1}^k \frac{1+i}{i} \right)
			= \log(1+k).
\end{align*}

If we set
\begin{equation*}
	w_k = \log\left( 1+ \frac{\log\left( 1+ \frac{1}{k} \right)}{1+\log k} \right),
		\quad\text{for $k\in\N$}
\end{equation*}
then, using Taylor's expansions again,
\begin{equation*}
	w_k
		\approx \frac{\log\left( 1+ \frac{1}{k} \right)}{1+\log k}
		\approx \frac{1}{k\log k}
\end{equation*}
and
\begin{align*}
	W(k)
		& = \sum_{i=1}^k \log\left( 1+ \frac{\log\left( 1+ \frac{1}{k} \right)}{1+\log k} \right)
			\\
		& = \log\left(\prod_{i=1}^k \frac{1+\log(1+i)}{1+\log i} \right)
			\\
		& = \log\bigl(1+\log(1+k)\bigr).
\end{align*}
The conclusion then follows by Proposition~\ref{prop:recall-loss}.
\end{proof}

\begin{figure*}[b]
  \centering
  \begin{subfigure}[b]{\linewidth}
  \centering
  \includegraphics[width=0.245\linewidth]{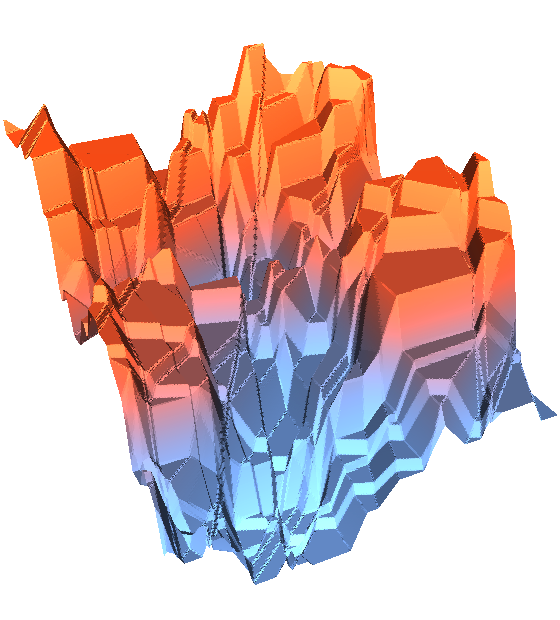}
  \includegraphics[width=0.245\linewidth]{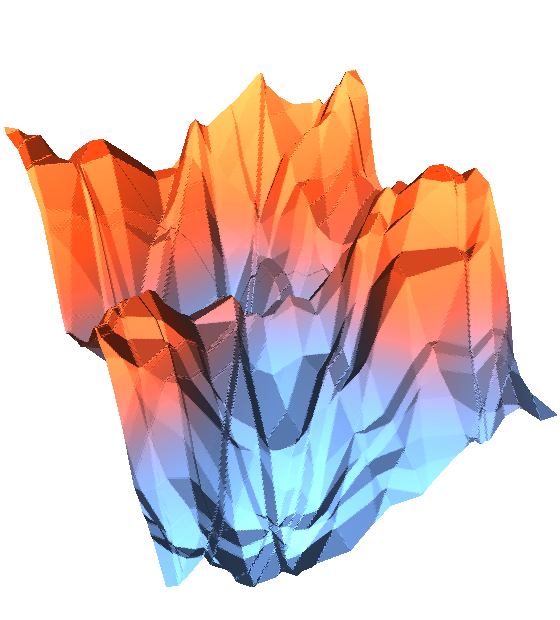}
  \includegraphics[width=0.245\linewidth]{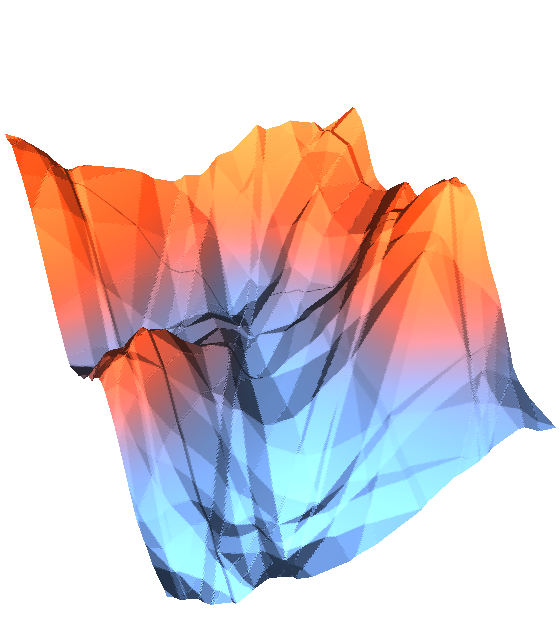}
  \includegraphics[width=0.245\linewidth]{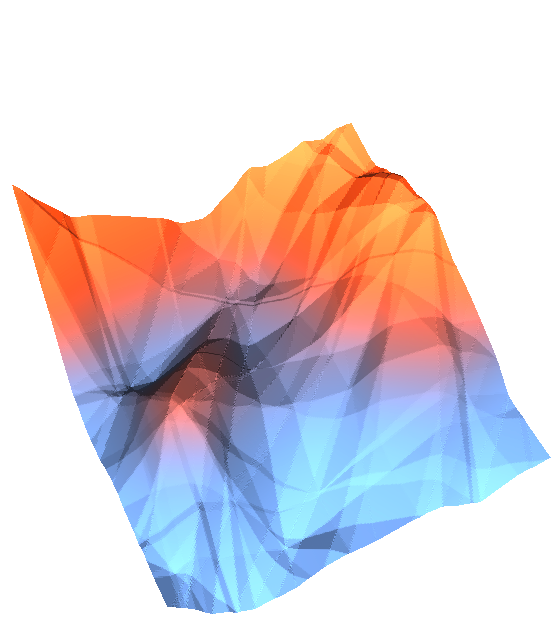}
  \caption{Ranking interpolation by \citep{blackbox-diff} for $\lambda=0.2, 0.5, 1.0, 2.0$.}
  \label{fig:f-lambda-all}
  \end{subfigure}
  \begin{subfigure}[b]{\linewidth}
  \centering
  \includegraphics[width=0.245\linewidth]{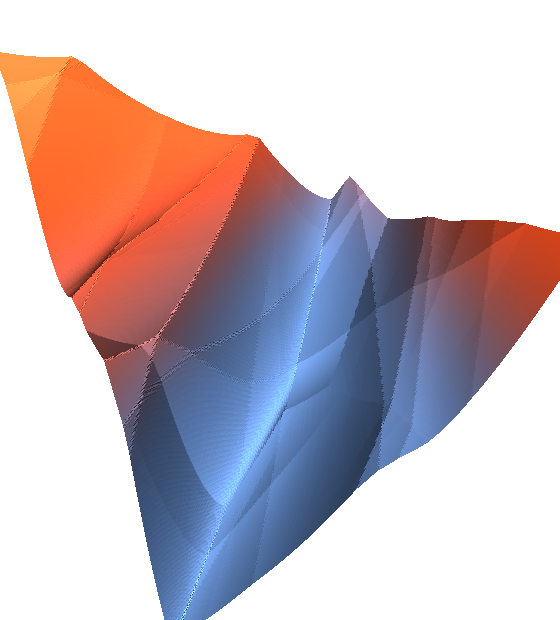}
  \includegraphics[width=0.245\linewidth]{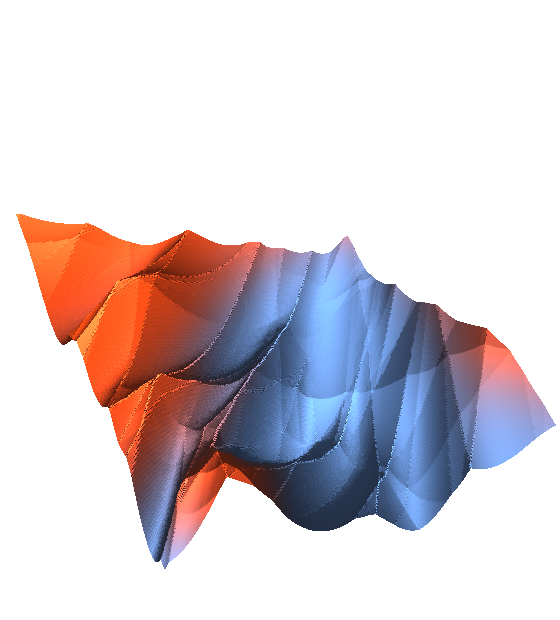}
  \includegraphics[width=0.245\linewidth]{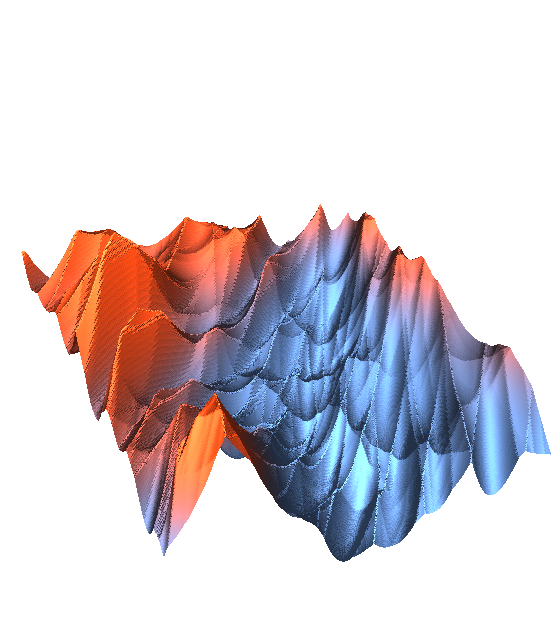}
  \includegraphics[width=0.245\linewidth]{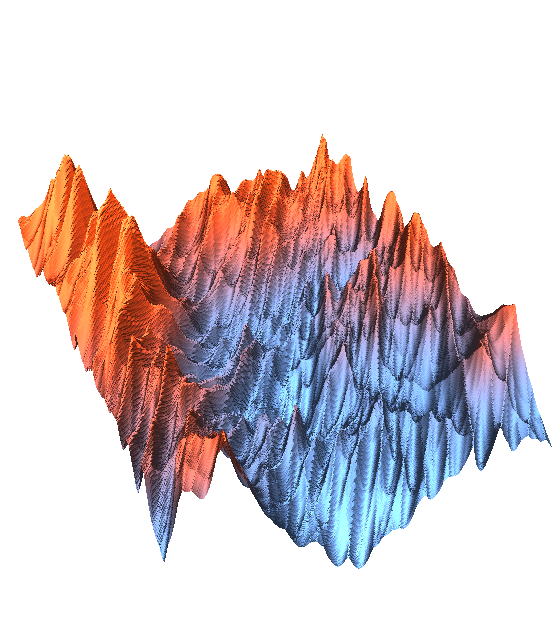}
  \caption{FastAp \citep{cakir2019deep} with bin counts $5,10,20,40$.}
  \label{fig:f-fast-all}
  \end{subfigure}
  \caption{Evolution of the ranking-surrogate landscapes with respect to their parameters.}
  \label{fig:}
	\vglue0.5in
\end{figure*}

\section{Ranking surrogates visualization}

\noindent
For the interested reader, we additionally present visualizations of smoothing
effects introduced by different approaches for direct optimization of
rank-based metrics. We display the behaviour of our approach using blackbox
differentiation \citep{blackbox-diff}, of FastAP  \citep{cakir2019deep}, and of
SoDeep \citep{engilberge2019sodeep}.

In the following, we fix a 20-dimensional score vector $w \in \R^{20}$ and a
loss function $L$ which is a (random but fixed) linear combination of the ranks
of $w$. We plot a (random but fixed) two-dimensional section of $\R^{20}$ of
the loss landscape $L(w)$.  In \fig{fig:f-orig} we see the true piecewise
constant function. In \fig{fig:f-lambda}, \fig{fig:f-deep} and \fig{fig:f-fast}
the ranking is replaced by interpolated ranking \citep{blackbox-diff}, FastAP
soft-binning ranking \citep{cakir2019deep} and by pretrained SoDeep LSTM
\citep{engilberge2019sodeep}, respectively. In \fig{fig:f-lambda-all} and
\fig{fig:f-fast-all} the evolution of the loss landscape with respect to
parameters is displayed for the blackbox ranking and FastAP.

\section*{Acknowledgement}

\noindent
We thank the International Max Planck Research School for Intelligent Systems
(IMPRS-IS) for supporting Marin Vlastelica and Claudio Michaelis. We
acknowledge the support from the German Federal Ministry of Education and
Research (BMBF) through the Tübingen AI Center (FKZ: 01IS18039B). Claudio
Michaelis was supported by the Deutsche Forschungsgemeinschaft (DFG, German
Research Foundation) via grant EC 479/1-1 and the Collaborative Research Center
(Projektnummer 276693517 -- SFB 1233: Robust Vision).

\begin{figure*}
	\centering
	\begin{subfigure}[b]{0.49\linewidth}
		\includegraphics[width=\textwidth]{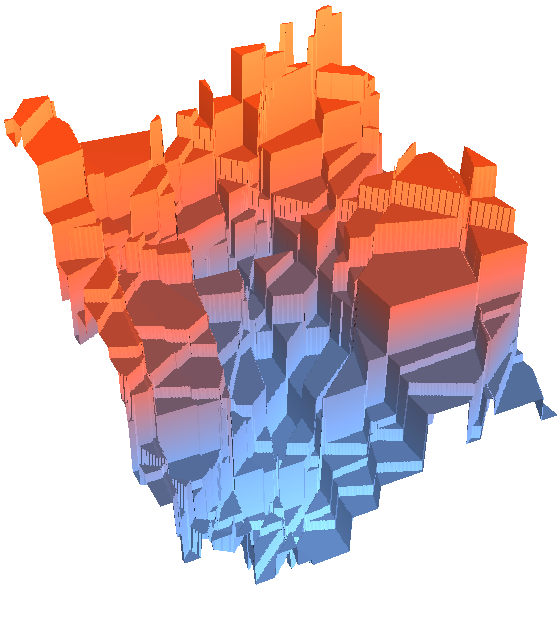}
		\caption{Original piecewise constant landscape}
		\label{fig:f-orig}
	\end{subfigure}
	\begin{subfigure}[b]{0.49\linewidth}
		\includegraphics[width=\textwidth]{f_lambda_05}
		\caption{Piecewise linear interpolation scheme of \citep{blackbox-diff}
			with $\lambda=0.5$}
		\label{fig:f-lambda}
	\end{subfigure}
	\\
	\begin{subfigure}[b]{0.49\linewidth}
		\includegraphics[width=\textwidth]{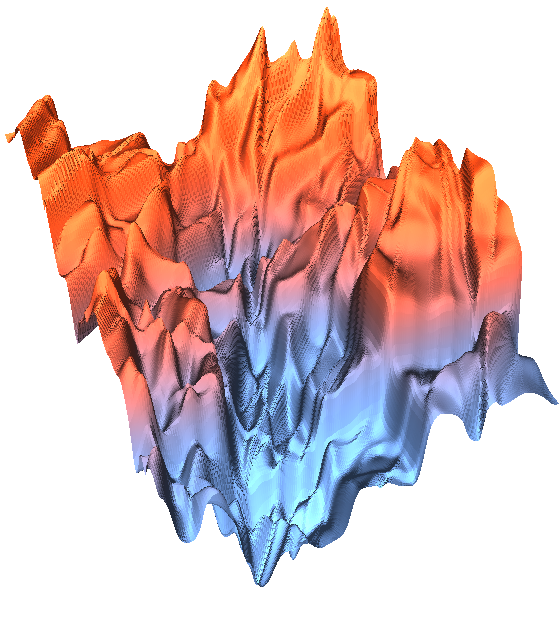}
		\caption{SoDeep LSTM-based ranking surrogate \citep{engilberge2019sodeep}}
		\label{fig:f-deep}
	\end{subfigure}
	\begin{subfigure}[b]{0.49\linewidth}
		\includegraphics[width=\textwidth]{f_fast_10}
		\caption{FastAP \citep{cakir2019deep} soft-binning with 10 bins.}
		\label{fig:f-fast}
	\end{subfigure}
	\caption{Visual comparison of various differentiable proxies for piecewise
		constant function.}
	\label{fig:f-ranks}
\end{figure*}

\end{document}